\documentclass{article}
\usepackage{fullpage}
\usepackage{amsmath}
\usepackage{amsthm,amsfonts,amssymb}
\usepackage{xspace}
\usepackage{bbm}
\usepackage{graphicx}
\usepackage{bbm}
\usepackage{hyperref}
\usepackage[noend]{algorithmic}
\usepackage[ruled,vlined]{algorithm2e}
\usepackage{natbib}

	\hypersetup{
			colorlinks=true,
			linkcolor=[rgb]{0,0,.5},
			urlcolor=[rgb]{0,0,.5},
			citecolor=[rgb]{0,0,.5},
			pdfstartview=FitH}

\usepackage{scrextend}

\usepackage{accents}

\usepackage{graphicx} 
\usepackage{booktabs} 

\usepackage{caption}
\usepackage{authblk}

\usepackage{xcolor}

\usepackage{pgfplots}

\pgfplotsset{compat=1.10}

\usepackage{tikz}

\usepackage{enumitem}
\usepackage[caption=false]{subfig}
\usepackage{mathtools}
\usetikzlibrary{quotes,positioning}

\newtheorem{theorem}{Theorem}[section]
\newtheorem{lemma}{Lemma}[section]

\newtheorem{definition}{Definition}[section]
\newtheorem{remark}{Remark}[section]

\newtheorem{obs}{Observation}[section]

\usepackage{xcolor}

\newif\ifcomment
\commentfalse

\ifcomment
\newcommand{\rv}[1]{\textcolor{red}{[RV: #1]}}
\newcommand{\cj}[1]{\textcolor{blue}{[CJ: #1]}}
\newcommand{\ar}[1]{\textcolor{purple}{[AR: #1]}}
\newcommand{\mmp}[1]{\textcolor{orange}{[MP: #1]}}
\else
\newcommand{\rv}[1]{}
\newcommand{\cj}[1]{}
\newcommand{\ar}[1]{}
\newcommand{\mmp}[1]{}
\fi

\newcommand{\ind}{\mathbbm{1}}

\newcommand{\E}{\mathbb{E}}
\newcommand{\cX}{\mathcal{X}}
\newcommand{\cY}{\mathcal{Y}}
\newcommand{\cH}{\mathcal{H}}
\newcommand{\cP}{\mathcal{P}}
\newcommand{\cG}{\mathcal{G}}
\newcommand{\cS}{\mathcal{S}}
\newcommand{\bmu}{\overline{\mu}}
\newcommand{\mk}{m_k}
\newcommand{\bmk}{\overline{m}_k}

\newcommand{\bma}{\overline{m}_a}

\newcommand{\tma}{\widetilde{m}_{a,\bmu}}
\newcommand{\tmk}{\widetilde{m}_{k,\bmu}}
\newcommand{\tth}{^\text{th}}

\newcommand{\kth}{k^{\text{th}}}

\newcommand{\bell}{\overline{\ell}}
\newcommand{\bQ}{\overline{Q}}
\newcommand{\cR}{\mathcal{R}}

\usepackage{thmtools} 
\usepackage{thm-restate}

\usepackage{amsmath}
\usepackage{pgfplots}

\title{Moment  Multicalibration for  Uncertainty Estimation}

\author[1]{Christopher Jung}
\author[2]{Changhwa Lee}
\author[3]{Mallesh M. Pai}
\author[1]{Aaron Roth}
\author[4]{Rakesh Vohra}
\affil[1]{University of Pennsylvania Department of Computer and Information Science}
\affil[2]{University of Pennsylvania Department of Economics}
\affil[3]{Rice University Department of Economics}
\affil[4]{University of Pennsylvania Department of Economics and Electrical and Systems Engineering}

\begin{document}
\maketitle

\begin{abstract}
    We show how to achieve the notion of  ``multicalibration''  from \citet{multicalibration} not just for means, but also for  variances and other higher moments. Informally, it means that we can find regression functions which, given a data point, can make point predictions not just for the expectation of its label, but for higher moments of its label distribution as well---and those predictions match the true distribution quantities when averaged not just over the population as a whole, but also when averaged over an enormous number of finely defined subgroups. It yields a principled way to estimate the uncertainty of predictions on many different subgroups---and to diagnose potential sources of unfairness in the predictive power of features across subgroups. As an application, we show that our moment estimates can be used to derive marginal prediction intervals that are simultaneously valid as averaged over all of the (sufficiently large) subgroups for which moment multicalibration has been obtained. 
\end{abstract}

\thispagestyle{empty} \setcounter{page}{0}
\clearpage

\section{Introduction}
Uncertainty estimation is fundamental to prediction and regression. Given a training set of labelled points $D \subseteq \cX \times [0,1]$ consisting of feature vectors $x \in \cX$ and labels $y \in [0,1]$, the standard regression problem is to find a function $\bmu:\cX\rightarrow [0,1]$ that delivers a good point estimate of $\mu(x) = \E[y | x]$. We also desire the \emph{variance} of the label distribution $\E[(y - \mu(x))^2|x]$ as a measure of the inherent uncertainty of a prediction. Higher central moments would yield even more information about this uncertainty which can be represented by \emph{prediction intervals}: An interval $[\ell(x), u(x)]$ that with high probability contains $y$, i.e., $\Pr_y[y \in [\ell(x), u(x)] | x] \geq 1-\delta$ for some $\delta \in (0,1)$. 

If the data are generated according to a parametric model as in the classic ordinary least squares setting, one can form confidence regions around the underlying model parameters, and translate these into both mean and uncertainty estimates about individual predictions. In non-parametric settings it is unclear how one should reason  about uncertainty. We typically observe each feature vector $x$ infrequently, and so we have essentially no information about the true distribution on $y$ conditional on $x$. One solution to this problem is to compute \emph{marginal} prediction intervals which  average over \emph{data points} $x$ to give guarantees of the form: $\Pr_{x,y}[y \in [\ell(x), u(x)] ] \geq 1-\delta$. This is the approach that is taken in the  \emph{conformal prediction} literature --- see e.g. \cite{conformal}.

Marginal prediction intervals, unlike prediction intervals, do {\em not} condition on $x$. They offer a promise not over the randomness of the label conditional on the features, but over an average over data points. To make the distinction vivid, imagine one is a patient with high blood pressure, and a statistical model asserts that  a certain drug will lower one's diastolic blood pressure  to between 70 and 80 mm Hg. If $[70, 80]$ were a 95\% prediction interval \emph{conditional on all of one's observable features}, then one could reason that over the unrealized randomness in the world, there is a 95\% chance that one's new blood pressure will lie in $[70,80]$. If $[70,80]$
is a 95\% marginal prediction interval, however, it means that  \emph{95\% of all patients who take the drug} will see their blood pressure decline to a level contained within the interval. Because the average is taken over a large, heterogeneous collection of people, the guarantee of the marginal prediction interval offers no meaningful promise to individuals. For example,
it is possible that patients that share one's demographic characteristics (e.g. women of Sephardic Jewish descent with a family history of diabetes) will tend to see their blood pressure elevated by the drug. 

This fundamental problem with uncertainty estimation in non-parametric settings is also a problem for mean estimation: what does it mean that a point prediction $\bmu(x)$ is an estimate of $\E[y | x]$ if we have no  knowledge of the distribution on $y$ conditional on $x$ (because we have observed no samples from this distribution)? A standard performance measure is \emph{calibration}  \citep{dawid1982well}, which similarly averages over data points: a predictor $\bmu$ is calibrated (roughly) if $\E_{(x,y)}[\bmu(x) - y | \bmu(x) = i] = 0$ for all predictions $i$: i.e. for every $i$,  conditioned on $x$ being such that the prediction $\bmu(x)$ was (close to) $i$, the expected outcome $y$ is also (close to) $i$. Just as with marginal prediction intervals, guarantees of calibration mean little to individuals, who differ substantially from the majority of people over whom the average is taken. 

\cite{multicalibration} proposed \emph{multicalibration}  as a way to interpolate between the (unattainable) ideal of being able to correctly predict $\E[y | x]$ for each $x$ and offering a guaranteed averaged over the entire data distribution. The  idea is to fix a large, structured set of (possibly overlapping) sub-populations ($\cG \in 2^X$). A predictor $\bmu$ is  multicalibrated if, informally, for all predictions $i$ and groups $G \in \cG$, $\E_{(x,y)}[\bmu(x) - y | \bmu(x) = i, x\in G] = 0$. Thus, $\bmu$ is calibrated not just on the overall population, but also simultaneously on many different finely defined sub-populations that one might care about (e.g. different demographic groups). \citet{multicalibration} show how to compute an approximately multicalibrated predictor $\bmu$ on all subgroups in $\cG$ that have substantial probability mass---we provide a high level description of their algorithm, which we use, below.

The main contribution of this paper is to show how to achieve what can loosely be termed multicalibration for higher moment estimates. We provide not just estimates $\bmu(x)$  of means ($\mu(x) = \mathbb{E}[y|x]$), but also estimates, $\bmk(x)$, for higher central moments, ($\mk(x) = \E[(y-\mu(x))^k|x]$) such that all of these forecasts are appropriately multicalibrated in a sense made precise below. This is useful for a number of basic tasks. One we briefly highlight is that it can help diagnose data iniquities: for example, if the set of collected features is much less predictive of the target label on certain demographic groups $G \in \cG$ this will necessarily manifest itself in multicalibrated moment predictions by having higher variance predictions on individual members of those populations. 

As an important application, we show that standard concentration inequalities which could be applied using the true moments of a distribution to obtain prediction intervals can also be applied using our  multicalibrated moment estimates. Doing so  produces intervals $[\ell(x), u(x)]$  for each data point that are \emph{simultaneously} valid marginal prediction intervals not just overall, but also conditioned on $x$ lying in any of the (sufficiently large) subgroups over which we are multicalibrated.  This allows one to interpret these prediction intervals as predicting something meaningful not just an average over all people, but --- simultaneously --- as averages over all of the people who were given the same prediction, across many finely defined subgroups (like women of Sephardic Jewish descent with a family history of diabetes). Note that because the groups $G \in \cG$ may overlap, a single individual can belong to many such groups, and can at her option interpret the prediction interval as averaging over any of them. 

\subsection{Overview of Our Approach and Results}

\subsubsection{Mean Multicalibration and Impediments to Extensions to Higher Moments}

We first review the algorithm of \citet{multicalibration}, recast in the framework in which we will conduct our analysis. We here elide some issues such as how we deal with discretization and how calibration error is parameterized --- see Section \ref{sec:prelims} for the formal model and definitions. Fix a feature space $\cX$, labels $\cY = [0,1]$, and an unknown distribution $\cP$ over $\cX \times \cY$. Given are sets $\cG \subseteq 2^{\cX}$, corresponding to sub-populations of interest. The goal is to construct a predictor, $\bmu:\cX\rightarrow \cY$, that is multicalibrated, i.e. calibrated on each group $G \in \cG$. This means that we want a predictor, $\bmu$, that is \emph{mean-consistent} on every set of the form $G(\bmu,i) = \{ x \in G : \bmu(x)=i\}$ for some $i$: in other words, for every such set $G(\bmu,i)$ we want $\E_{(x,y) \sim \cP}[\bmu(x) - y] | x \in G(\bmu,i)] = 0$. We describe the algorithm as if it has direct access to the true distribution $\cP$, and defer for now a description of how to implement the algorithm using a finite sample. 

It is helpful to conceive of the task as a zero-sum game between two players: a ``(mean) consistency player'', and an ``audit player'' who knows the true distribution $\cP$.  The consistency player chooses a predictor  $\bmu$, and the audit player, given a predictor, attempts to identify a subset $S$ of $\cX$ on which the predictor is not mean consistent.%
\footnote{Here, and in what follows, we adopt the convention that $G$ refers to a group in $\cG$, while $S$ refers to any generic subset of $\cX$.}
Given a pair of choices, the corresponding cost (which the consistency player wishes to minimize and the audit player wishes to maximize) is the absolute value difference between the average prediction of the consistency player and the average expected label on the subset $S$ identified by the audit player. The value of this game is  $0$, since the consistency player can obtain perfect consistency using the true conditional label distribution $\bmu(x) = \mathbb{E}[y|x]$. The algorithm of \citet{multicalibration} can be interpreted as solving this zero sum game by simulating repeated play, using online gradient descent for the consistency player, and ``best response'' for an audit player, who stops play if there are no remaining sets $S = G(\bmu,i)$ witnessing violations of multicalibration. This works because by linearity of expectation, we can formulate the game so that the consistency player's utility function is \emph{linear} in her individual predictions $\bmu(x)$.  A formal description and proof of correctness can be found in Section \ref{subsec:mean-calibration}.

There are two---related---impediments to extending this approach to higher moments, i.e., finding predictors $\bmk (x) \approx \mk(x) =  \E[(y-\E[y|x])^k|x]$, that are ``consistent'' with $\cP$ on many sets. The first of these is definitional---what do we mean by ``consistent'' for higher moments? The second is algorithmic---given a definition, how do we achieve it? Both are impediments because, unlike means, higher moments are not linear functionals of the distribution.  A consequence is that moments for $k > 1$ do not combine linearly in the way expectations do. In particular for $S= S_1 \cup S_2$ where $S_1$ and $S_2$ are disjoint, $\E[(y-\E[y|x\in S])^k|S] \neq \Pr (x \in S_1|S)\E[(y-\E[y|x \in S_1])^k|S_1]  + \Pr (x \in S_2|S)\E[(y-\E[y|x \in S_2])^k|S_2]$. 
It is therefore silly to require that moment predictions $\bmk(x)$ satisfy the same ``average consistency'' condition asked of means: i.e. we cannot demand that the population variance on the subset of the population on which we predict variance $v$ be $v$, because this is not a property that the true moments $\mk(x)$ satisfy. Consider, for example, a setting in which there are two types of points, $x_1$ and $x_2$. The true distribution is uniform over $\{(x_1,0), (x_2, 1)\}$ (and so in particular the label $y$ is deterministically fixed by the features). We  have that for all $k > 1$, $\mu(x_1) = 0, \mu(x_2) = 1$, and  $\mk(x_0) = \mk(x_1) = 0$. Nevertheless, the variance over the set of points on which the  true distribution satisfies $\mk(x) = 0$ is $1/4$, not $0$. We cannot ask that our ``moment calibrated'' predictors satisfy properties violated by the true distribution, because we would have no guarantee of feasibility --- and our ultimate goal in multicalibration is to find a set of mean and moment predictors that are indistinguishable from the true distribution with respect to some class of tests.

\subsubsection{Mean Conditioned Moment Multicalibration and Marginal Prediction Intervals}

A key observation (Observation \ref{observation:mixture}) is that higher moments \emph{do} linearize over sets that have the same mean: in other words, if we have $S = S_1 \cup S_2$ for disjoint $S_1$ and $S_2$ such that $\E[y | x \in S_1] = \E[y | x \in S_2]$, then, it follows that $\E[(y-\E[y|x\in S])^k|S] = \Pr (x \in S_1|S)\E[(y-\E[y|x \in S_1])^k|S_1]  + \Pr (x \in S_2|S)\E[(y-\E[y|x \in S_2])^k|S_2]$. An implication of this is that the true distribution does satisfy what we term \emph{mean-conditioned moment multi- calibration}. Namely, if for a fixed $k > 1$ we define for each set $G \in \cG$ and each pair of mean and $k\tth$ moment values $i,j$ the sets: $G(\mu,\mk,i,j) = \{x \in G : \mu(x) = i, \mk(x) = j\}$, then we  have \emph{both} mean consistency: $\E[(y - i) | x \in G(\mu,\mk,i,j)] = 0$ \emph{and} moment consistency:  $\E[(y-i)^k - j | x \in G(\mu,\mk,i,j)] = 0$ over these sets. Therefore, we require the same condition to hold for our mean and moment  predictors $\bmu$ and $\{\bma\}_{a=1}^k$: namely that simultaneously for every $a$, that over each of the sets $G(\bmu,\bma,i,j)$, the true label mean should be $i$ and the true label $a$-th moment should be $j$. In other words, if we have a set of predictors that are mean conditioned moment multicalibrated, then an individual who receives a particular mean and (e.g.) variance prediction can be assured that amongst all the people who received the same mean and variance prediction \emph{even averaged over any of the possibly large number of sub-groups $G$ of which the individual is a member}, the true mean and variance are faithful to the prediction.  

Section \ref{sec:intervals} demonstrates a key application of mean-conditioned moment-multicalibrated estimators: They can be used in place of real distributional moments to derive prediction intervals. Given moments of a random variable $X$, a standard way to derive concentration inequalities for $X$ is by using the following inequality for any even moment (for $k = 2$ this is Cheybychev's inequality): 
\[
    \Pr[|X - \mu(X)| \geq t] \leq \frac{\E\left[(X - \mu(X) )^k\right]}{t^k}.
\]
 If $X$ is the label distribution conditional on features $x$, this yields the prediction interval: \[\Pr_y\left[y \in \left[\mu(x)-\left(\frac{\mk(x)}{\delta}\right)^{1/k},\mu(x)+\left(\frac{\mk(x)}{\delta}\right)^{1/k}\right] \middle\vert x\right] \geq 1-\delta.\]
In Section \ref{sec:intervals}  we show that if we have a mean-conditioned moment-multicalibrated pair $(\bmu,\bmk)$, we can replace the true mean and moments in the derivation of this prediction interval, and get  \emph{marginal} prediction intervals, which are valid not just averaged over all points, but simultaneously as averaged over all point that received the same prediction within any of the groups within which we are mean-conditioned moment multicalibrated. In other words, for all $G \in \cG$ and for all $i,j$:
$$\Pr_{(x,y) \sim P}\left[y \in \left[\bmu(x)-\left(\frac{\bmk(x)}{\delta}\right)^{1/k},\bmu(x)+\left(\frac{\bmk(x)}{\delta}\right)^{1/k}\right] \middle\vert x \in G(\bmu,\bmk,i,j)\right] \geq 1-\delta.$$

\subsubsection{Achieving Mean Conditioned Moment Multicalibration}
What is the difficulty with finding sets of predictors $(\bmu,\{\bma\}_{a=2}^k)$ such that simultaneously  each pair $(\bmu,\bma)$ are mean-conditioned moment multicalibrated? It is that moments do not have the linear structure that means do. Hence, the zero-sum game formulation we describe for mean-multicalibration cannot be applied directly. A na\"ive approach (which fails, but which will be a useful sub-routine for us) is to first train a mean-multicalibrated predictor $\bmu$, and then define ``pseudo-moment'' labels for each $x$ as $\tmk(x) = (y-\bmu(x))^k$. Since these are constant values,  we can then use the algorithm for mean multicalibration to achieve ``pseudo-moment calibration with respect to $\bmu$'' --- i.e. mean consistency on each set $G(\bmu,\bmk,i,j)$ with respect to our pseudo-moment labels $\tmk(x)$. By itself this doesn't guarantee any sort of ``moment consistency,'' but we show in Section \ref{subsec:pseudo-moment-calibration} that if we can:
\begin{enumerate}
\item Find moment predictors $\bmk$ that satisfy pseudo-moment calibration with respect to $\bmu$, \emph{and}
\item Our mean predictor $\bmu$ satisfies mean consistency on every set of the form $G(\bmu,\bmk,i,j)$,
\end{enumerate}
then, the pair $(\bmu,\bmk)$ will satisfy mean-conditioned moment calibration. 

The difficulty is that these two requirements are circularly defined. Once we have a \emph{fixed} mean predictor $\bmu$, we can use a gradient descent procedure to find moment predictors $\{\bma\}_{a=2}^k$ that are pseudo-calibrated with respect to $\bmu$. However, we also require our mean predictor to be mean consistent on the sets  $G(\bmu,\bma,i,j)$, which are undefined until we fix our moment predictors $\{\bma\}_{a=2}^k$. Section \ref{subsec:mean-conditioned-moment-calibration} resolves the circularity by using an alternating descent procedure that toggles between updating $\bmu$ and $\{\bma\}_{a=2}^k$, each aiming for a mean calibration target that is defined with respect to the other. We prove that this alternating gradient descent procedure is guaranteed to converge after only a small number of rounds. 

Finally, we show in Section \ref{sec:finite} how to implement our algorithm using a finite sample from the distribution and furnish sample complexity bounds, in a way analogous to \citet{multicalibration}. The sample complexity bounds are logarithmic in the number of groups $|\cG|$ that we wish to be multicalibrated with respect to, and polynomial in our desired calibration error parameters and the number of moments $k$  with which we wish to achieve mean-conditioned moment multicalibrated predictors. In particular, because dependence on $|G|$ is only logarithmic, we can satisfy mean-conditioned moment-multicalibration on an exponentially large collection of intersecting sets $\cG$ from just a polynomial sample of data from the unknown population distribution. Note, however, that despite our polynomial dependence on $k$, the natural scale of the $k$'th moment decreases exponentially in $k$, and so to obtain non-trivial approximation guarantees for $k$'th moments with polynomial sample complexity, we should think of taking $k$ at most logarithmic in the relevant parameters of the problem. See Theorem \ref{thm:finite-wrapper} and Corollary \ref{cor:finite-wrapper} for details. Our running time  scales polynomially with our approximation error parameters, the number of moments $k$ with which we wish to be multicalibrated, and the running time of solving learning problems over $\cG$ (which is at most linear in $|\cG|$, but can be much faster). See Theorems \ref{thm:finite-wrapper} and \ref{thm:agnostic} for details.  In other words, our algorithms are ``oracle efficient'' in the sense that if we have a subroutine for solving learning problems over $\cG$, then we can use it to solve mean-conditioned moment-multicalibration problems with at most polynomial overhead. In theory, for almost every interesting class $\cG$, learning over $\cG$ is hard in the worst case --- but oracle efficiency has proven to be a useful paradigm in the design of learning algorithms (especially in the fairness in machine learning literature --- see e.g. \citep{reductions,gerrymandering,multicalibration,multiaccuracy}) because in practice we have extremely powerful heuristics for solving complex learning problems. Moreover, this kind of oracle efficiency is the best running time guarantee that we can hope for, because as shown by \citet{multicalibration}, even mean-multicalibration is as hard as solving arbitrary learning problems over $\cG$ in the worst case.

\subsection{Additional Related Work}\label{sec:related}

 \emph{Calibration} as a means of evaluating  forecasts of expectations dates back to \citet{dawid1982well}. This literature focuses on a simple online forecasting setting, motivated by weather prediction problems: in a sequence of rounds, nature chooses the probability of some binary event (e.g. rain), and a forecaster predicts a probability of that event. \citet{dawid1982well} shows that a Bayesian forecaster will always be subjectively calibrated (i.e. he will believe himself to be calibrated). \citet{FV98} show that there exist  \emph{randomized} forecasters that can asymptotically satisfy calibration against arbitrary sequences of outcomes (this is impossible for deterministic forecasters \citep{oakes1985self}). These papers focus on the online setting, because simple calibration is trivial in a batch/distributional setting: simply predicting the mean outcome on every point satisfies calibration. Within this literature, the most related works are \citet{lehrer2001any} and  \citet{sandroni2003calibration}, which give very general asymptotic results that are able to achieve (mean) multicalibration as a special case. \citet{lehrer2001any}, operating in the sequential online setting, asks for calibration to hold not just on the entire sequence of realized outcomes, but on countably many infinite \emph{sub-sequences} (e.g. the set of all computable subsequences). He proves that there exists an online forecasting algorithm which can asymptotically achieve this. \citet{sandroni2003calibration} extend this result to subsequences which can be defined in terms of the forecasters predictions as well.  Both of these papers operate in a setting that is general enough to encode the constraint of mean  multicalibration (by encoding the features of datapoints in the ``state space'') even in an online, adversarial setting --- albeit not in a computationally or sample efficient way. In contrast, \citet{multicalibration}, who define the notion of mean multicalibration, give an algorithm for achieving it in a batch distributional setting ---  in a much more computationally and sample efficient manner than could have been achieved by applying the machinery of \cite{lehrer2001any,sandroni2003calibration}.  Recently, \citet{individualcalibration} gave a notion of ``individual level'' (mean) calibration, defined over the randomness of the forecaster, that is valid conditional on individual data points (i.e. without needing to average over a population). They provide promising empirical results, but the theoretical guarantees of predictors satisfying this notion do not provide non-trivial information about a data distribution, because (as the authors note) their notion of individual calibration can be satisfied without observing any data.  

\citet{multicalibration} also proposed the notion of ``multi-accuracy,'' a weaker notion than multicalibration which asks for a predictor $\bmu$ that satisfies mean consistency on each set $G \in \cG$, but not on sets $G(\bmu,i)$. \citet{multiaccuracy} gave a practical algorithm for achieving multi-accuracy, and a promising set of experiments suggesting that it could be used to correct for error disparities between different demographic groups on realistic data sets, without sacrificing overall accuracy. \citet{dwork2019learning} propose notions of fairness and evidence consistency for ranking individuals by their ``probability of success'' when historical data only records binary outcomes: they show that their proposed notions  are closely related to multicalibration of the probability predictions implicitly underlying the rankings. \citet{multiUC} prove uniform convergence bounds for multicalibration error over hypothesis classes of bounded complexity. In our paper, as in \citet{multicalibration}, we learn over hypothesis classes that are only implicitly defined by the set of groups $\cG$, and so we bound generalization error in the same manner that \cite{multicalibration} do, rather than using uniform convergence arguments. 

Conformal prediction is similarly motivated to calibration, but is focused on finding marginal prediction intervals rather than mean estimates: see e.g. \citet{conformal} for an overview of this literature. Finding marginal prediction intervals on its own (i.e. when prediction intervals only have to be valid on average over the entire population) is easy in the batch/distributional setting, and so just as with the calibration literature, the conformal prediction literature is primarily focused on the online setting in which predictions must be made as points arrive. The most closely related paper related to this literature is  \citet{barber2019limits} who also study the batch distributional setting, and also aim to find marginal prediction intervals which hold not just over the entire population, but on a collection $\cG$ of more finely defined sub-populations. \citet{barber2019limits} obtain prediction intervals of this sort by using a holdout set method from conformal prediction: roughly speaking, they compute empirical $1-\delta$ coverage intervals on each set $G \in \cG$ in the holdout set, and then for an individual $x$, select the \emph{widest} such interval amongst all groups $G$ that contain $x$, which is a very conservative choice.  The algorithm given by \citet{barber2019limits}  relies on explicit enumeration of groups $G \in \cG$ over the holdout set.

There are also several  papers  in the ``fairness in machine learning'' literature (in addition to \cite{multicalibration,multiaccuracy}), which are similarly motivated by replacing coarse statistical constraints with constraints that come closer to offering individual guarantees: see \citet{fairsurvey} for a survey. \citet{gerrymandering,subgroup} propose to learn classifiers which equalize statistical measures of harm like false positive or negative rates across a very large number of demographic subgroups $G \in \cG$, and give practical algorithms for this problem by solving a zero-sum game formulation using techniques from no-regret learning. \citet{cbawareness} give algorithms for satisfying a notion of metric fairness which similarly enforces constraints averaged over a large number of subgroups $G \in \cG$. \citet{RY18}  define a PAC-like version of the individual fairness notion of \citet{awareness} and prove generalization bounds showing how to achieve their notion out of sample on all sufficiently large groups of individuals.  \citet{aif} show how to equalize statistical measures of harm like false positive rates across \emph{individuals} --- when the rates in question are defined over the randomness of the problem distribution and the classifier.  \citet{fairbandits,fairbandits2} propose an individual-level notion of ``weakly meritocratic fairness'' that can be satisfied in bandit learning settings whenever it is possible to compute confidence or prediction intervals around individual labels. They analyze the parametric setting, when actual (conditional) prediction and confidence intervals are possible --- but the techniques from our paper could be used for learning in the assumption-free setting  (with a slightly weaker notion of fairness) using marginal prediction intervals.

\section{Preliminaries}
\label{sec:prelims}
Let $\cX$ be the domain of features, $\cY = [0,1]$ the label domain, and $\cP$ the true (unknown) probability distribution over $\cX \times \cY$.\footnote{Our approach applies for both finite and infinite feature domains. If $\cX$ is uncountably infinite, define an associated measure space, and $\cP$ is a countably additive probability measure on this space. We omit the associated notation since it will have no use in what follows.} Let $\cP_\cX$ refer to the induced marginal distribution on $\cX$ and define $\cP_\cY$ analogously.  Going forward, we refer to the associated random variables with capital letters (e.g. $X$, $Y$), and realizations with lowercase letters ($x$, $y$). 

Let $\cG \subseteq 2^\cX$ be a collection of subsets of $\cX$,\footnote{If $\cX$ is uncountably infinite, then $\cG$ is a collection of measurable, computable sets. We abuse notation and write $2^\cX$ to denote this.} and for each $G \in \cG$, let $\chi_G$ denote that associated indicator function, i.e. $\chi_G(x) = 1 \iff x \in G$. For implementation purposes, we assume that each indicator function $\chi_G(x)$ can be computed by a polynomially sized circuit.
\footnote{Our algorithm in the end will need to manipulate these indicator functions. We might imagine e.g. that $\cG$ is the hypothesis class of some learning algorithm for a binary prediction problem, and that the functions $\chi_G(x)$ are particular hypotheses from this class --- e.g. linear threshold functions.}

\begin{definition}\label{def:moments}
Given the true distribution $\cP$, we write
\[
\mu = \E_{\cP}[y],
\]
and its $k$th central moment is:
\[\mk = \E_{\cP}\left[(y-\mu )^k\right]. \]
Given a set $S \subseteq \cX$, we abuse notation and write 
\[
\mu(S) = \E_{\cP}[y|x \in S] \quad\text{and}\quad\mk(S) = \E_{\cP}\left[\left( y-\mu(S)\right)^k| x \in S \right]
\]
for the conditional mean and $k^{\text{th}}$ central moment of labels on the distribution conditional on $x \in S$. 
\end{definition} 

We are given $n$ independent draws from $\cX \times \cY$ according to distribution $\cP$, denoted  $D=\{(x_b, y_b)\}_{b=1}^n$. The goal is to predict means and higher moments of $\cY| \cX$, i.e. to construct functions $\bmu: \cX \to [0,1],$ (we shall refer to this as a mean predictor) and $\bmk:\cX \to [0,1]$ (analogously, $k\tth$-moment predictor)---as $\cY$ is the unit interval, means and moments also lie in the unit interval. 

To define calibration, we need to reason about all points that receive a particular prediction. For real valued predictors, this can be a measure zero set. One solution is to to restrict attention to predictors that are discretized to lie on the grid $G_m = \{\frac{1}{2m}, \frac{3}{2m} ,\ldots,\frac{2m-1}{2m}\}$, for some (large) number $m$. If one were to do this, the discretization parameter $m$ would be coupled to the error one could ultimately obtain: since it may be inevitable to suffer error at least $1/2m$ if one is restricted to making predictions on a discrete grid. Alternately, one can define calibration by ``bucketing'' real valued predictions into $m$ buckets of width $\frac1m$ each. This allows us to treat $m$ (a parameter controlling the fineness of our calibration constraint) as an orthogonal parameter to our calibration error. To that end, given a set $S\subseteq \cX$, mean  predictor $\bmu$, and some $i \in [m]$, define \[S(\bmu, i) \equiv \left\{x \in S: \left\vert \bmu(x) - \frac{2i-1}{2m}\right\vert \le \frac{1}{2m} \right\}\]
to be the set of points in $S$ whose mean predictions fall into the $i\tth$ bucket, i.e. $[ \frac{2i-1}{2m}- \tfrac{1}{2m}, \frac{2i-1}{2m}+ \tfrac{1}{2m}]$.
Analogously, define \[S(\bmu, \bmk, i,j) \equiv \left\{x \in S: \left\vert \bmu(x) - \frac{2i-1}{2m}\right\vert \le \frac{1}{2m} , \left\vert \bmk(x) - \frac{2j-1}{2m}\right\vert \le \frac{1}{2m} \right\}\] to be the set of points in $S$ that receive mean predictions in the $i\tth$ bucket  \emph{and} $\kth$ moment predictions in the $j\tth$ bucket. Given mean and $\kth$ moment predictors $\bmu$ and $\bmk$, and any set $S \subseteq \cX$ we write 
\[
\bmu(S) = \E_{\cP }[\bmu(x)|x \in S] \quad\text{and}\quad \bmk(S) = \E_{\cP }\left[\bmk(x)|x \in S\right],
\]
i.e. $\bmu(S)$ is the average mean prediction of $\bmu$ when $x$'s are drawn according to the true distribution, $\cP $, conditional on $x \in S$, and $\bmk(S)$ is the analogous quantity for $k$'th moment predictions. 

To be clear, we will maintain the convention for means and higher moments that quantities  with bars refer to predictions ($\bmu, \bmk$) and unmodified notation ($\mu, \mk$) refer to true (unknown) population values.

\begin{definition}[Consistency]
Call a mean predictor $\bmu$  $(\alpha, \epsilon)$-mean consistent on a set $S$ if 
\[
\left\vert \mu\left(S\right) - \bmu\left(S\right)  \right\vert \le \frac{\alpha }{\cP_\cX( S)} + \epsilon.
\]

Similarly, a moment predictor $\bmk$ is called $(\alpha, \epsilon)$-moment consistent on a set $S$ if: 
\[
\left\vert \mk\left(S\right) - \bmk\left(S\right)  \right\vert \le \frac{\alpha }{\cP_\cX( S)} + \epsilon.
\]

When $\epsilon=0$, we say $\bmu$ is $\alpha$-mean consistent and $\bmk$ is $\alpha$-moment consistent. Note that $(\alpha, \epsilon)$-mean consistency implies $(\alpha + \epsilon)$-mean consistency.
\end{definition}

\begin{remark}
Our notion of consistency on a set $S$ corresponds to  error that smoothly degrades with the size (measure) of the set $S$. This is essential to giving out of sample guarantees. \citet{multicalibration} handles this slightly differently, by giving uniform guarantees, but only for sets that have  measure at least $\gamma$.  Our approach of giving smoothly parameterized  error guarantees for all sets is only stronger (up to a reparameterization of $\alpha \leftarrow \alpha\gamma$), and makes the analysis of our algorithms  more transparent because it corresponds more directly to the guarantees they achieve.
\end{remark}

The following simple observation will be useful in understanding our approach. 
\begin{obs}
\label{observation:mixture} 
Let $\cP$ be a mixture distribution over $m$ component distributions $\cP_\ell$ with mixture weights $w_\ell\geq 0$, $\sum_{\ell=1}^m w_\ell = 1$. Let $\mu_\ell, {\mk}_\ell$ be the mean and $k^{\text{th}}$ moment associated with $\cP_\ell$.  Then:
\begin{align*}
&\mk =\sum_{\ell=1}^m w_\ell \left(\sum_{a=0}^k {k \choose a} \left(\mu_\ell - \mu\right)^{k-a}   m_{a \ell} \right).
\intertext{If the mixture variables have the same mean, i.e. $\mu_\ell = \mu$ for all $\ell$, then, the above expression reduces to:}    
&\mk =\sum_{\ell=1}^m w_\ell {\mk}_\ell.
\end{align*}
\end{obs}

Observation \ref{observation:mixture} highlights the key challenge: unlike means, higher moments combine non-linearly over mixtures. That is to say, that although $\bmk(S)$ is defined to be an average over the values $\bmk(x)$ for $x \in S$, $\mk(S)$ is \emph{not} an average over the values $\mk(x)$ for $x \in S$ for $k > 1$. Observation \ref{observation:mixture} also makes clear what we are trying to exploit in defining mean-\emph{conditioned} moment calibration: $\mk(S)$ \emph{is} an average over the values $\mk(x)$ for $x \in S$ whenever $\mu(x)$ is constant over $S$. 

We are now ready to define calibration, which asks for mean and moment consistency on particular sets defined by the mean and moment predictors themselves:
\begin{definition}[Calibration]\label{def:calibrated} Fix a set $S \subseteq \cX$ and a true distribution $\cP$.
\begin{enumerate}
\item A mean predictor $\bmu$ is $(\alpha,\epsilon)$-mean calibrated on a set $S$ if it is $(\alpha, \epsilon)$-mean consistent on every set $S(\bmu,i)$, i.e. if  for each $i \in [m]$:
\[
\left\vert \mu\left(S(\bmu, i)\right) - \bmu(S(\bmu, i)) \right\vert \le \frac{\alpha }{\cP_\cX( S(\bmu, i))} + \epsilon.
\]
Again, if $\epsilon=0$, we say $\bmu$ is $\alpha$-mean calibrated.
\item Predictors $(\bmu, \bmk)$ are $(\alpha, \beta, \epsilon)$-mean-conditioned-moment calibrated on a set $S$ if they are $(\alpha, \epsilon)$-mean and $(\beta, \epsilon)$-moment consistent on every set $S(\bmu, \bmk, i, j)$, i.e.  if for every $i,j \in [m]$:
\begin{align*}
&\left\vert \mu\left(S(\bmu, \bmk, i, j)\right) - \bmu\left(S(\bmu, \bmk, i, j)\right) \right\vert \le \frac{\alpha}{\cP_\cX( S(\bmu, \bmk, i, j))} + \epsilon,\\ 
\text{and }&\left\vert \mk\left(S(\bmu, \bmk, i, j)\right) - \bmk\left(S(\bmu, \bmk, i, j)\right) \right\vert \le \frac{\beta}{\cP_\cX( S(\bmu, \bmk, i, j))}+\epsilon.
\end{align*}
If $\epsilon=0$, we say $(\bmu, \bmk)$ are $(\alpha,\beta)$-mean-conditioned-moment calibrated.
\end{enumerate}
We say that $\bmu,\bmk$ are $(\alpha,\epsilon)$-multicalibrated and $(\alpha,\beta,\epsilon)$-mean-conditioned-moment multicalibrated with respect to (a collection of sets) $\cG$ if they are $(\alpha,\epsilon)$-mean calibrated and $(\alpha,\beta,\epsilon)$-mean conditioned moment calibrated respectively on every $G \in \cG$.
\end{definition}

\begin{remark}
Observe that by construction, the true feature conditional mean and moment functions $\mu(x),\mk(x)$ are mean-conditioned-moment multicalibrated on every collection of sets $G$. We can view the goal of multicalibration as coming up with mean and moment predictors $\bmu,\bmk$ that are almost indistinguishable from the true distributional means and moments, with respect to a class of consistency checks defined by $\cG$. Note that it is only because we have defined our goal as mean \emph{conditioned} moment calibration that the true moments $\mk(x)$ of the distribution satisfy these consistency conditions, which are defined as expectations. 
\end{remark}

We highlight the difference between calibration and consistency on a given set $S$ in terms of mean prediction $\bmu$; an analogous discussion applies to higher moments. Consistency requires that the  prediction $\bmu(x)$, averaged over $x$'s in $S$ according to the conditional distribution, approximately equals the true label average $\mu(S)$. It doesn't impose a similar requirement on subsets of $S$. Therefore, a predictor consistent on $S$ will be correct on average for the set $S$ but could be systematically biased for each prediction in $S$ that it makes. 

Calibration on $S$ requires, for every prediction $i \in [m]$, that $\bmu$ is consistent on the set $S(\bmu, i)$. That is to say it ensures consistency on every subset of $x$'s in $S$ on which the predictor $\bmu$ makes predictions in some some fixed bucket $i$. Exact calibration implies exact consistency, but the reverse is not true.

\section{Achieving Mean Conditioned Moment Multicalibration}\label{achieving}

\subsection{Mean Multicalibration}
\label{subsec:mean-calibration}
We summarize an algorithm to achieve mean multicalibration. It is a modest extension to the one in \citet{multicalibration} that accommodates arbitrary distributions over a possibly infinite domain and arbitrary initializations. We present it in somewhat greater generality than needed for mean-calibration, because our final algorithm in Section \ref{subsec:mean-conditioned-moment-calibration} needs to achieve mean consistency on more sets than are required for mean calibration alone.

For intuition, consider the following mini-max problem, which captures a more difficult problem than mean multicalibration (as there is no restriction at all on the sets $S$):
\begin{align*}
\min_{\bmu: \cX \to [0,1]}\max_{\substack{S \subseteq \cX,\\ \lambda \in \{-1,1\}}} \lambda \cdot \cP_{\cX}(S)\cdot \left( \bmu(S) - \mu(S)  \right).
\end{align*}

We can associate a zero-sum game with this mini-max problem by viewing the minimization player as a \emph{consistency}  player who must commit to a mean predictor $\bmu$, and viewing the maximization player as an \emph{auditor} who attempts to identify sets $S$ on which the consistency player fails to be mean consistent. Observe that the inclusion of the measure term $\cP_{\cX}(S)$ in the objective makes the learner's utility function linear in her individual predictions $\bmu(x)$. There is a strategy for the consistency player that would guarantee her a payoff of $0$---or in other words, would guarantee consistency on all possible sets $S$: she could simply set $\bmu(x) = \E[y | x]$. This establishes the value of the game, but of course it requires knowledge of $\cP$.  Given only a finite sample of the data, we will be unable to determine  $\E[y | x]$ for all $x$, and so this strategy is not implementable.

One way to solve our problem absent knowledge of the distribution is to allow the consistency player to play online gradient descent \citep{Zin03} on the set of mean predictors over rounds $t$, and to allow the auditor to ``best respond'' at every round, by exhibiting a set $S$ corresponding to a large consistency violation\footnote{Because the objective function of our game weights the consistency violations $\bmu(S)-\mu(S)$ by the measure of the set $\cP_{\cX}(S)$, these violations are linear functions of the individual predictions $\bmu(x)$. Thus it suffices to run gradient descent over the space of individual predictions $\cX$, rather than the space of all possible functions $\bmu:\cX\rightarrow [0,1]$.}. This is guaranteed to converge quickly to an approximate equilibrium of the game: i.e. a mean predictor satisfying approximate consistency on all sets. If the auditor limits herself to choosing sets $S(\bmu^t, i)$ corresponding to mean calibration, then we converge quickly to approximate mean calibration.  Here we give a direct analysis of a general gradient descent procedure of the sort we need, in terms of the sets that the auditor happens to choose during this interaction. For finite support distributions $\cP$, this bound could be derived directly from the regret bound of online projected gradient descent \citep{Zin03} or from the analysis of the similar algorithm  in  \citet{multicalibration}. We reproduce a direct analysis in the Appendix to match the theorem statement we want for distributions which may have infinite support. (Note that for such distributions the mean predictor will have to be maintained implicitly). In Algorithm \ref{alg:meantemplate}, after each gradient update, we project $\bmu^t$ back into the set of functions with range $[0,1]$ using an $\ell_2$ projection. Because squared $\ell_2$ distance is linearly separable, it can be accomplished by a simple coordinate-wise operation which we write as $\text{project}_{[0,1]}(x) = \min(\max(x, 0), 1)$.

\begin{algorithm}[H]
\SetAlgoLined
\begin{algorithmic}
\STATE Start with an arbitrary initial mean predictor $\bmu^1:\cX\rightarrow [0,1]$
\FOR{$t=1, \dots, T$}
	\STATE Audit player plays some $S^t \subseteq \cX, \lambda^t \in \{-1,1\}$
	    \STATE $\bmu^{t+1}(x) = \begin{cases}
	\text{project}_{[0,1]}\left(\bmu^t(x)  - \eta \lambda^t\right) & \text{if } x \in S^t,\\
	\bmu^t(x) & \text{otherwise.}
	\end{cases}$
\ENDFOR
\end{algorithmic}
\caption{Projected Gradient Descent$(\eta)$ for $\bmu$}
\label{alg:meantemplate}
\end{algorithm}
\bigskip

\begin{restatable}{lemma}{lemregret}
\label{lem:boundregret}
For any initial mean predictor $\bmu^1 \in \cX \to [0,1]$ and any sequence of $(S^t, \lambda^t)_{t=1}^T$, Algorithm \ref{alg:meantemplate} satisfies:
\[
\sum_{t=1}^T \lambda^t \cP_{\cX}(S^t) \left(\bmu^t(S^t) - \mu(S^t)\right) \le \frac{1}{2\eta} + \frac{\eta}{2} \sum_{t=1}^T \cP_{\cX}(S^t).
\]
\end{restatable}

\bigskip

The proof is in the Appendix. A direct consequence of the bound in Lemma \ref{lem:boundregret} is that, when interacting with a consistency player who uses gradient descent with learning rate $\eta = \alpha$, an auditor will be able to find sets that fail to be $\alpha$-mean consistent for at most $1/\alpha^2$ many rounds.\footnote{A somewhat better bound is achievable by using a non-uniform learning rate that depends on the measure of the sets $S^t$ chosen by the auditor; we use a uniform learning rate throughout this paper for clarity.} The following theorem is a direct consequence of Lemma \ref{lem:boundregret} --- its short proof is in the Appendix. 

\begin{restatable}{theorem}{thmogdfewrounds}
\label{thm:mean-calibration-convergence}
Set $T=\frac{1}{\alpha^2}-1$ and $\eta=\alpha =\frac{1}{\sqrt{ T+1}}$ in Algorithm \ref{alg:meantemplate}. Assume that for every $t \in [T]$, 
\[
\lambda^t \left(\bmu^t(S^t) - \mu(S^t) \right) \ge \frac{\alpha}{\cP_{\cX}(S^t)},
\]
Then, for every $S \subseteq \cX$, we have
\[
    \left\vert \bmu^{T+1}(S) - \mu(S) \right\vert \le \frac{\alpha}{\cP_{\cX}(S)}.
\]
\end{restatable}
In particular, if the auditor selects sets $G(\bmu^t,i)$ that fail to satisfy approximate mean consistency whenever they exist, then we quickly converge to a mean-multicalibrated predictor. Either we reach a state in which $\bmu^t$ is approximately mean consistent on every set $G(\bmu^t,i)$ before $T$ rounds, in which case we are done, or after $T$ rounds, the conclusion of Theorem \ref{thm:mean-calibration-convergence} implies not only that we are approximately mean-multicalibrated with respect to $\cG$, but that we are approximately mean-consistent on every set.

\subsection{Pseudo-Moment Consistency}
\label{subsec:pseudo-moment-calibration}

In this section we make a simple observation: Algorithm \ref{alg:meantemplate} from Section \ref{subsec:mean-calibration} for achieving mean consistency and calibration did not depend on \emph{any} properties of the labels $y$. It would have worked equally well had we invented an arbitrary label for each datapoint $x$, and asked for mean consistency with respect to that label. Using this observation, we consider a (na\"ive, and incorrect) attempt at achieving calibration for higher moments --- but one that will be a useful subroutine in our final algorithm. Recall that $\mk(S) = \E[(y-\mu(x))^k | x \in S]$. If we have a mean predictor $\bmu(x)$, it is therefore tempting to imagine that each point $x$ is associated with an alternative label $\tilde y(x) = \tmk(x)$, where:
\[
\tmk(x) =  \E\left[\left(y - \bmu(x) \right)^k \middle\vert x\right].
\] 
We could then use the algorithm from Section \ref{subsec:mean-calibration} to construct an predictor $\bmk$ that was \emph{mean} multicalibrated with respect to these labels. We refer to the property of being mean consistent with respect to the moment-like labels $\tmk(x)$ as  ``pseudo-moment-consistency'':

\begin{definition}[Pseudo-Moment-Consistency]
Fixing a mean predictor $\bmu$, define the $k\tth$ \emph{pseudo-moment labels} to be $\tmk(x) =  \E\left[\left(y - \bmu(x) \right)^k \middle\vert x\right]$. A moment predictor $\bmk$ is $(\beta,\epsilon)$-\emph{pseudo-moment-consistent} on a set $S$, with respect to a mean predictor $\bmu$ if
\[
\left\vert \bmk(S) - \tmk(S) \right\vert \le \frac{\beta}{\cP_\cX(S)} + \epsilon
\]
We simply say $\beta$-pseudo-moment consistent if the predictor is $(\beta,0)$-pesudo-moment-consistent. 
\end{definition}

We can achieve pseudo-moment consistency using the following gradient descent procedure, analogous to Algorithm \ref{alg:meantemplate}.

\begin{algorithm}[H]
\SetAlgoLined
\begin{algorithmic}
\STATE Start with an \emph{arbitrary} initial pseudo-moment predictor $\bmk^1:\cX\rightarrow [0,1]$
\FOR{$t=1, \dots, T$}
	\STATE Audit player plays some $R^t \subseteq \cX, \psi^t \in \{-1,1\}$
	    \STATE $\bmk^{t+1}(x) = \begin{cases}
	\text{project}_{[0,1]} \left(\bmk^t(x)  - \eta \psi^t\right) & \text{if } x \in R^t,\\
	\bmk^t(x) & \text{otherwise.}
	\end{cases}$
\ENDFOR
\end{algorithmic}
\caption{Projected Gradient Descent$(\eta)$ for $\bmk$}
\label{alg:momenttemplate}
\end{algorithm}

In particular, we  obtain the following theorem, whose proof is deferred to the Appendix.

\begin{restatable}{theorem}{thmpseudomoment}
\label{thm:pseudo-moment-calibration}
Let $T=\frac{1}{\beta^2}-1$ and $\eta=\frac{1}{\sqrt{ T+1}}=\beta$ in Algorithm \ref{alg:momenttemplate}, and fix any mean predictor $\bmu$, which defines the function $\tmk(x)$. Assume that for every $t \in [T]$, 
\[
\left\vert \bmk^t(R^t) - \tmk(R^t) \right\vert \ge \frac{\beta}{\cP_\cX(R^t)},
\]
Then, for every $R \subseteq \cX$, we have
\[
\left\vert \bmk(R) - \tmk(R) \right\vert \le \frac{\beta}{\cP_\cX(R)}.
\]
i.e. $\bmk$ is $\beta$-pseudo-moment-consistent on every set $R$.
\end{restatable}
Now, a guarantee of  ``pseudo-moment-consistency''  is really a guarantee of \emph{mean} consistency with respect to ``moment-like'' labels $\tmk(x)$, and does \emph{not} correspond to moment consistency. This is because moments $\mk$ for $k > 1$ don't combine linearly the way means do: recall Observation \ref{observation:mixture}. But also recall from Observation \ref{observation:mixture} that higher moments \emph{do} happen to combine linearly if we average only over points that share the same mean.  

 We take advantage of this to prove the following key lemma: if we achieve pseudo-moment consistency on all sets $G(\bmu,\bmk,i,j)$ (for $G\in \cG$) with respect to a mean predictor $\bmu$ that happens also to be mean-consistent on all sets $G(\bmu,\bmk,i,j)$, then, the pair of predictors is in fact approximately mean-conditioned moment multicalibrated with respect to $\cG$.

\begin{lemma}
\label{lem:fine-consist-mean-conditioned-moment-cal}
Assume $\bmu$ is such that for all $G \in \cG$ and $i,j \in [m]$, $\bmu$ is $\alpha$-mean consistent on every set $G(\bmu, \bmk, i, j))$:
\begin{align*}
&\left\vert \bmu(G(\bmu, \bmk, i, j)) - \mu(G(\bmu, \bmk, i, j)) \right\vert \le \frac{\alpha}{\cP_\cX(G(\bmu, \bmk, i, j))}.
\shortintertext{Assume also that $\bmk$ is $\beta$-pseudo-moment-consistent with respect to $\bmu$ on every set $G(\bmu, \bmk, i, j))$ for  $G \in \cG$ and $i,j \in [m]$:}
&\left\vert \bmk(G(\bmu, \bmk, i, j)) - \tmk(G(\bmu, \bmk, i, j)) \right\vert \le \frac{\beta}{\cP_\cX(G(\bmu, \bmk, i, j))}.
\shortintertext{Then, for every $G \in \cG$, $i,j \in [m]$, we have}
&\left\vert \bmk(G(\bmu, \bmk, i, j)) - \mk(G(\bmu, \bmk, i, j)) \right\vert \le \frac{\beta + k\alpha}{\cP_{\cX}(G(\bmu, \bmk, i, j)} + \frac{k}{m}. 
\end{align*}
This implies in particular that  $(\bmu, \bmk)$ are $(\alpha, \beta', \epsilon)$-mean-conditioned moment multicalibrated with respect to $\cG$, for $\beta' = \beta+ k\alpha$ and $\epsilon= \frac{k}{m}$.
\end{lemma}
\begin{proof}
Fix $G \in \cG$ and $i,j \in [m]$ and let $S \equiv G(\bmu, \bmk, i, j)$. Because  $\bmu$ is $\alpha$-mean consistent on $S$, we have that:
\begin{align}
\left\vert \mu(S) - \bmu(S) \right\vert \le \frac{\alpha}{\cP_\cX(S)}.\label{eqn:fine-consist-mean-conditioned-moment-calibration-assump}
\end{align}
We can use this to bound the difference between the true moment $\mk(S)$ and the pseudo-moment $\tmk(S)$ on $S$. First, note that: 
\begin{align*}
\mk(S) &= \E_{\cP}\left[( y - \mu(S))^k \middle\vert x \in S\right],\\
& = \E_{\cP} \left[ \left[\left( y - \bmu(x) \right) + \left(\bmu(x) - \mu(S)\right) \right]^k \middle\vert x \in S\right].
\end{align*}

We will make use of the following fact:
\begin{lemma}
\label{lem:arith}
For any $a,b \in [0,1]$, $|a^k-b^k |\leq k |a-b|$.
\end{lemma}
\begin{proof}
Observe that:
\begin{equation*}
|a^k-b^k |= \left|(a-b)\left(\sum_{\ell =0}^{k-1} a^{\ell}b^{k-1 -\ell}\right)\right| 
\leq  |a-b| \vert k(\max(a,b))^{k-1} \vert \leq k |a-b|. \qedhere \end{equation*}
\end{proof}

Finally, we conclude that:
\begin{align*}
\left|\mk(S) - \tmk(S) \right| &= \left|\E_{\cP} \left[\left(\left( y - \bmu(x) \right) + \left(\bmu(x) - \mu(S)\right) \right)^k- \left( y - \bmu(x) \right)^k \middle\vert x \in S\right]\right| \\
&\leq k \E_{\cP}\left[ \left\vert \bmu(x) - \mu(S)\right\vert \middle| x \in S\right]\\
& \leq k \left(\E_{\cP}\left[ \left\vert \bmu(x) - \bmu(S)\right\vert \middle| x \in S\right] +  \left\vert \bmu(S) - \mu(S)\right\vert  \right)\\
 &\leq k \left(\frac{1}{m}+ \frac{ \alpha}{\cP_{\cX}(S)}  \right).
\end{align*}
The first inequality follows from Lemma \ref{lem:arith} with $a = \left( y - \bmu(x) \right) + \left(\bmu(x) - \mu(S)\right)$ and $b =  y - \bmu(x)$. The final inequality follows from  \eqref{eqn:fine-consist-mean-conditioned-moment-calibration-assump} (mean consistency) together with the fact that  $\bmu(x)$ falls within a bucket of width $\frac1m$ for any $x \in S$ (recall that by definition,  $S = G(\bmu, \bmk, i, j)$), and so does $\bmu(S)$

Finally, because $\bmk$ is $\beta$-pseudo-moment consistent on $S$ with respect to $\bmu$ we can invoke the triangle inequality to conclude:
\begin{align*}
    \left\vert \bmk(S) - \mk(S) \right\vert
    &\le \left\vert \bmk(S) - \tmk(S) \right\vert + \left\vert \tmk(S) - \mk(S) \right\vert,  \\
    &\le \frac{\beta}{\cP_\cX(S)}+ k \left(\frac{1}{m} + \frac{ \alpha}{\cP_{\cX}(S)} \right).  \qedhere
\end{align*}
\end{proof}

Lemma \ref{lem:fine-consist-mean-conditioned-moment-cal} reduces the problem of finding mean-conditioned-moment multicalibrated predictors $(\bmu,\bmk)$ to the problem of finding a pair of predictors $(\bmu,\bmk)$ satisfying mean-consistency and pseudo-moment-consistency on the sets $G(\bmu,\bmk,i,j)$. It is unclear how to do this, because these conditions have a circular dependency: pseudo-moment consistency of $\bmk$ with respect to $\bmu$ is not defined until we have fixed a mean predictor $\bmu$, because the ``labels'' $\tmk(x)$ with respect to which pseudo-moment consistency is defined depend on $\bmu$. On the other hand, the sets $G(\bmu,\bmk,i,j)$ on which $\bmu$ must satisfy mean consistency are not defined until we fix the moment predictor $\bmk$. The next section is devoted to resolving this circularity and finding predictors satisfying the conditions of Lemma \ref{lem:fine-consist-mean-conditioned-moment-cal}.

\subsection{Mean-Conditioned Moment Multicalibration}
\label{subsec:mean-conditioned-moment-calibration}

We arrive at the last block upon which our main result rests: an alternating gradient descent procedure that on any distribution finds a mean multicalibrated predictor $\bmu$ together with moment predictors $\{\bma\}_{a=2}^k$ such that each pair $(\bmu,\bma)$ is approximately mean-conditioned moment multicalibrated on $\cG$. We continue, for clarity's sake, to assume access to the underlying distribution $\cP$, and postpone to Section \ref{sec:finite} the details of implementing this approach with a polynomially sized sample of points. Our strategy is to obtain a set of predictors that together satisfy the hypotheses of Lemma \ref{lem:fine-consist-mean-conditioned-moment-cal}: mean consistency and pseudo-moment-consistency on every set of the form  $G(\bmu, \bma, i, j))\subseteq G \in \cG$, $1 < a \leq k$, and $i,j \in [m]$. We have already seen in Section \ref{subsec:mean-calibration} that for a \emph{fixed} collection of sets, a simple gradient-descent procedure can obtain mean consistency on each of the sets. Section \ref{subsec:pseudo-moment-calibration} demonstrates that for a \emph{fixed} mean predictor $\bmu$, a similar gradient descent procedure can obtain pseudo-moment-consistency with respect to $\bmu$ on each set $G(\bmu, \bma, i, j))$. Our algorithm simply alternates between these two procedures. In rounds $t$, we maintain hypothesis predictors $\bmu^t,\{\bma^t\}_{a=2}^k$. In alternating rounds, we perform updates of gradient descent using Algorithm \ref{alg:mean} to arrive at a mean predictor $\bmu^{t}$ that has taken a step towards consistency on sets $G(\bmu^t, \bma^{t-1}, i, j)$, and then using the newly updated mean predictor $\bmu^t$, run Algorithm \ref{alg:moment} to obtain moment predictors $\bma^t$  that obtain pseudo-moment-consistency with respect to $\bmu^t$ on all sets $G(\bmu^t, \bma^t, i, j)$. This is coordinated via a wrapper algorithm, Algorithm \ref{alg:wrapper}. We prove this alternating procedure terminates after $1/\alpha^2-1$ many rounds and outputs predictors $\bmu,\{\bma\}_{a=2}^k$ that are jointly mean-conditioned moment-multicalibrated.

\begin{algorithm}[H]
\begin{algorithmic}
    \STATE $\bmu(x) = \begin{cases}
    	\text{project}_{[0,1]}(\bmu(x)  - \alpha \lambda) & \text{if } x \in S, \\
    	\bmu(x) & \text{otherwise.}
    	\end{cases}$
    \STATE return $\bmu$
\end{algorithmic}
\caption{MeanConsistencyUpdate$(\bmu,S,\lambda)$}
\label{alg:mean}
\end{algorithm}

\begin{algorithm}[H]
\begin{algorithmic}[1]
\STATE define pseudo-moment labels $\tma(x)  = \E\left[\left( y-\bmu(x)\right)^a \middle| x\right]$ for all $x$
    \WHILE{$\exists R = G(\bmu, \bma, i, j)$ for some $G \in \cG$, $i, j \in [m]$ s.t. $\left|\bma(R) - \tma(R) \right| \geq \frac{\beta}{\cP_\cX(R)}$}\label{algline:pseudo-moment}
            \STATE $\psi = \text{sign}(\bmk(R) - \tmk(R))$
            \STATE $\bma(x) = \begin{cases}
            	\text{project}_{[0,1]}(\bma(x)  - \beta \psi) & \text{if } x \in R \\
            	\bma(x) & \text{otherwise.}
            	\end{cases}$
\ENDWHILE
\STATE return $\bmk$
\end{algorithmic}
\caption{$\text{PseudoMomentConsistency}(a, \beta, \bmu, \bma, \cG)$}
\label{alg:moment}
\end{algorithm}

\begin{algorithm}[H]
\begin{algorithmic}[1]
\STATE initialize $\bmu^1(x) = 0$ for all $x$
    \STATE for all $1 < a \leq k$, initialize $\bma^1(x) = 0$ for all $x$
\STATE $t=1$
\WHILE{$\exists S^t = G(\bmu^t, i)$ or $S^t = G(\bmu^t, \bma^t, i, j)$ for some $G \in \cG$, $i, j \in [m]$, $1 < a \leq k$ s.t. $\left|\bmu^{t}(S^t) - \mu(S^t)\right| \ge  \frac{\alpha}{\cP_{\cX}(S^t)}$}\label{algline:wrapper}
    \STATE $\lambda^t = \text{sign}(\bmu(S^t) - \mu(S^t))$
    \STATE $\bmu^{t+1} = \text{MeanConsistencyUpdate}( \bmu^{t}, S^t,\lambda^t)$
    \FOR{$a=2, \dots, k$}
        \STATE $\bma^{t+1} = \text{PseudoMomentConsistency}(a, \beta, \bmu^{t+1}, \bma^{t},\cG)$.
    \ENDFOR
    \STATE $t = t+1$
\ENDWHILE
\STATE return $(\bmu^t, \{\bma^t\}_{a=2}^k)$
\end{algorithmic}
\caption{$\text{AlternatingGradientDescent}(\alpha,\beta,\cG)$}
\label{alg:wrapper}
\end{algorithm}

\begin{theorem}
\label{thm:wrapper-distr-known}
Let $T$ be the final iterate  $t$ of Algorithm \ref{alg:wrapper} (i.e. its output is $(\bmu^T, \{\bma^T\}_{a=2}^k)$.  Algorithm \ref{alg:wrapper} has the following guarantees:
\begin{enumerate}
\item \textbf{Total Iterations}: The algorithm halts. The final iterate $T$ is s.t. $T \leq \frac{1}{\alpha^2}-1$. The total number of gradient descent update operations is at most $\left(\frac{1}{\alpha^2}-1\right)\left(1 + (k-1) \left(\frac{1}{\beta^2}-1\right)\right)$.
\item \textbf{Mean multicalibration}: Output $\bmu^T$ is $\alpha$-mean multicalibrated with respect to $\cG$.
\item \textbf{Mean Conditioned Moment multicalibration}: For every $a \in \{2, \dots, k\}$, the pair $(\bmu^T, \bma^T)$ is $(\alpha, \beta + a\alpha, {a \over m})$-mean-conditioned moment-multicalibrated with respect to $\cG$.
\end{enumerate}
\end{theorem}
\begin{proof}
We prove each guarantee in turn.
\medskip

\noindent \textbf{Total Iterations}: Every step $t$ of the while loop in Algorithm \ref{alg:wrapper} performs a gradient descent update using $\text{MeanConsistencyUpdate}$ on a pair $(\lambda^t, S^t)$ such that: \[
\lambda^t \left(\bmu^t(S^t) - \mu(S^t) \right) \ge \frac{\alpha}{\cP_{\cX}(S^t)}.
\]
By Theorem \ref{thm:mean-calibration-convergence}, this process can continue for at most $T \leq \frac{1}{\alpha^2}-1$ many iterations. Within each iteration $t$ of the loop, the algorithm makes one call to $\text{PseudoMomentConsistency}$ for each $1 < a \leq k$ for a total of $(k-1)$ calls per iteration. Each of these calls performs at most $\frac{1}{\beta^2}-1$ iterations of gradient descent, by Theorem \ref{thm:pseudo-moment-calibration}.

\medskip

\noindent \textbf{Mean multicalibration}: Suppose for a contradiction that Algorithm \ref{alg:wrapper} terminates at $t = T$ with output $\bmu^T$ which is not mean multicalibrated, i.e. there exists a set $S \equiv G(\bmu^T,i)$ for some $G \in \cG$, $i \in [m]$ such that $|\bmu^t(S)-\mu(S)| \geq \frac{\alpha}{\cP_{\cX}(S)}$. Then, by construction of the while loop in Algorithm \ref{alg:wrapper}, $T$ cannot be the final iterate of $t$. 

\medskip

\noindent \textbf{Mean Conditioned Moment multicalibration}: The While loop in Algorithm \ref{alg:wrapper} will continue as long as there exists a set $S^t \equiv G(\bmu^t, \bma^t, i, j)$ for some $G \in \cG, i,j \in [m]$ such that $|\bmu^t(S^t)-\mu(S^t)| \geq \frac{\alpha}{\cP_{\cX}(S^t)}$. Hence we can conclude that at termination, $\bmu^T$ is $\alpha$-mean consistent on every set $G(\bmu^T, \bma^T, i, j)$ for some $G \in \cG, i,j \in [m]$. Moreover, during the final iteration, for each $1 < a \leq k$, $\bma^T$ was constructed by running $\text{PseudoMomentConsistency}(a, \beta, \bmu^{T}, \bma^{T-1},\cG)$. Therefore, by Theorem \ref{thm:pseudo-moment-calibration} we know that $\bma^T$ is $\beta$-pseudo-moment consistent on every set $G(\bmu^T, \bma^T, i, j)$. To see this, note that if $\text{PseudoMomentConsistency}$ runs for $\frac{1}{\beta^2}-1$ many rounds, then it is $\beta$-pseudo-moment consistent on \emph{every} set. On the other hand, the only way it can halt before that many rounds (by construction of the halting condition in its While loop) is if $\bma^T$ is $\beta$-pseudo-moment consistent on every set $G(\bmu^T, \bma^T, i, j)$.  

Therefore, $\bmu^T$ and $\{\bma^T\}_{a=2}^k$ jointly satisfy the conditions of Lemma \ref{lem:fine-consist-mean-conditioned-moment-cal}. It follows from the Lemma that they are mean-conditioned moment-multicalibrated at the desired parameters.
\end{proof}

\section{Implementation with Finite Sample and Runtime Guarantees}\label{sec:finite}
In Section \ref{achieving} we analyzed a version of our algorithm as if we had direct access to the true distribution, $\cP$. In particular, in both Algorithm \ref{alg:moment} and Algorithm \ref{alg:wrapper}, access to $\cP$ was needed in (only) two places. First, to identify a set $S^t$ such that $\left|\bmu^{t}(S^t) - \mu(S^t)\right| \ge  \frac{\alpha}{\cP_{\cX}(S^t)}$. Second, to identify a set $R$ such that $\left|\bma(R) - \tma(R) \right| \geq \frac{\beta}{\cP_\cX(R)}$.
In this section, we show how to perform these operations approximately by using a small finite sample of points drawn from $\cP$, and hence to obtain a finite sample version of our main result together with sample complexity and running time bounds.

There are two issues at play here: the first issue is purely statistical: how many \emph{samples} are needed to  execute the two checks needed to implement our algorithm? Our finite sample algorithm will essentially use a sufficiently large fresh sample of data at every iteration to guarantee uniform convergence of the quantities to be estimated over all of the sets that must be checked at that iteration. The second issue is computational: even if we have enough samples so that we can check in-sample quantities as proxies for the distributional quantities we care about, what is the running time of our algorithm? We are performing gradient descent in a potentially infinite dimensional space, and so we cannot explicitly maintain the weights $\bmu^t(x),\bma^t(x)$ for all $x$. Instead, we maintain these weights \emph{implicitly} as a weighted linear combination of the indicator functions for each of the sets $S^t$, $R$, used to perform updates (recall that we have assumed that each set $G \in \cG$ can be represented by an indicator function computed by a polynomially sized circuit, so we have concise implicit representations of every set that our algorithm must manipulate). Ostensibly one must exhaustively enumerate the collection of sets $S, R$, for which our algorithm must perform some check (in fact their indicator functions), which takes time that scales with $m^2\cdot |\cG|$. We first focus on the statistical problem, showing that the number of \emph{samples} needed to implement our algorithm is small, and then we observe that if we have an \emph{agnostic learning algorithm} for $\cG$, we can use it to replace exhaustive enumeration. In both cases---although the details differ---we handle these issues in largely the same way they were handled by \citet{multicalibration}, so many of the proofs and calculations will be deferred to the Appendix. 

Finally, we remark that it is essential that we draw a fresh sample of $n$ data points each time we try to find a set for consistency violation because $\ell$, $\bell$, as well as the collection of sets $\cS$ that we are auditing, are not fixed \emph{a priori} but change adaptively (i.e. as a function of the data) between rounds. Due to the adaptive nature of the statistical tests that need to be performed, we cannot simply union bound over these queries. We remark that we could have applied  adaptive data analysis techniques (see e.g. \citep{ada1,ada2,ada3}) to partially re-use the data, which would save a quadratic factor in the sample complexity (or for finite data domains, an exponential improvement in some of the existing parameters, at the cost of an additional dependence on $\log|\cX|$ by using the private multiplicative weights algorithm of \citet{PMW}). This idea is applied  in \citet{multicalibration}; it applies here in the same manner; interested readers can refer to  \citet{multicalibration}.

\subsection{Sample Complexity Bounds and a Finite Sample Algorithm via Exhaustive Group Enumeration}
\label{enumeration}
First, recall that pseudo-moment consistency is mean consistency with respect to the artificially created label $\tmk(x)=(y-\bmu(x))^k$. To avoid needless repetition, we focus on achieving mean consistency for an arbitrary label defined by a label function $\ell(x,y)$ with a predictor $\bell(x)$. Then, auditing for mean consistency for $\bmu$ and pseudo-moment consistency for $\bmk$ follows by setting 
\begin{align*}
  \bell(x) = \bmu(x) \quad&\text{and}\quad \ell(x,y) = y  \\
  \bell(x) = \bmk(x) \quad&\text{and}\quad \ell(x,y) = (y-\bmu(x))^k
\end{align*}
for mean consistency and pseudo-moment consistency respectively. For economy of notation set,
\[
\bell(S) = \E_\cP[\bell(x)|x\in S]
\quad\text{and}\quad \ell(S) = \E_\cP[\ell(x,y) | x \in S],
\]
for all $S \subseteq \cX$.
For any set $S \subseteq \cX$, given a dataset $D$, we refer to $D_S$ as the subset of the data where the corresponding points lie in $S$ (i.e. $D_S = \{(x,y) \in D: x \in S\}$). If dataset $D_S$ has $n'$ points each drawn independently from $\cP$ conditional on $x \in S$, we can  appeal to the Chernoff bound (Theorem \ref{thm:chernoff}) to argue that empirical averages must be close to their expectations.

We can also appeal to the Chernoff bound to argue that when we sample data points from $\cP$, the number of points that fall into some set $S$ (e.g. $G(\bmu, \bmk, i, j)$) scales roughly with $n \cP_\cX(S)$ (Lemma \ref{finitemeasurelem}).

Throughout the execution of our algorithm, we need to audit the current mean and moment estimators for $\alpha$-mean consistency violations. This is important, because the analysis of the running time of the algorithm (e.g. the fact that it converges after at most  $T=\frac{1}{\alpha^2}-1$ iterations) relies on making a minimum amount of progress guaranteed by $\alpha$-mean \emph{inconsistency}.  In the next lemma, we provide a condition that can be checked using empirical estimates which guarantees $\alpha$-mean inconsistency (on the true, unknown distribution) whenever the sample is appropriately close to the distribution; it follows from two applications of a Chernoff bound that this approximate closeness condition will occur with high probability. We encapsulate this empirical check in Algorithm \ref{alg:auditor-single-set}.

\begin{algorithm}[H]
\begin{algorithmic}
\IF{$n' > 0$ and $\left|\frac{1}{n'}\sum_{b=1}^{n'} \bell(x_b) - \frac{1}{n'} \sum_{b=1}^{n'} \ell(x_b, y_b)\right| - 2\sqrt{\frac{\ln(\frac{2}{\delta})}{2n'}} \geq  \frac{\alpha}{\frac{n'}{n} - \sqrt{\frac{\ln(\frac{2}{\delta})}{2n}}}$}
        \STATE $\lambda = \text{sign}(\frac{1}{n'}\sum_{b=1}^{n'} \bell(x_b) - \frac{1}{n'} \sum_{b=1}^{n'} \ell(x_b, y_b))$
        \STATE Output $YES, \lambda$ (A Consistency Violation has been found)
\ELSE
\STATE Output $No$
    \ENDIF
\end{algorithmic}
\caption{Auditor$(\ell, \bell, \alpha, \delta, \{(x_b, y_b)\}_{b=1}^{n'})$}
\label{alg:auditor-single-set}
\end{algorithm}

\begin{definition}
\label{def:chernoff-bounds-closeness}
 Fix any set $S\subseteq \cX$. Given a set of $n$ data points $D$ and its associated $D_S=\{(x_b,y_b)\}_{b=1}^{n'}$, we say that $D$ is approximately close to $\cP$ with respect to $(S,\ell,\bell)$, if the following inequalities hold true:
\begin{subequations}
\begin{align}
&n' > 0 \nonumber\\
&\left\vert \frac{n'}{n} - \cP_\cX(S) \right\vert \le \sqrt{\frac{\ln(\frac{2}{\delta})}{2n}}\label{eqn:chernoff1}\\
&\left\vert \frac{1}{n'}\sum_{b=1}^{n'} \bell(x_b) - \bell(S) \right\vert \le \sqrt{\frac{\ln(\frac{2}{\delta})}{2n'}}\label{eqn:chernoff2}\\
&\left\vert \frac{1}{n'}\sum_{b=1}^{n'} \ell(x_b,y_b) - \ell(S) \right\vert \le \sqrt{\frac{\ln(\frac{2}{\delta})}{2n'}}\label{eqn:chernoff3}
\end{align}
\end{subequations}
\end{definition}

\begin{restatable}{lemma}{auditorlem}
\label{lem:auditor}
Fix any set $S \subseteq \cX$. If dataset $D$ is approximately close to $\cP$ with respect to $(S,\bell, \ell)$, then we have
\[
\textrm{Auditor}(\ell, \bell, \alpha, D_S)=(\text{YES},\lambda) \implies \left|\bell(S) - \ell(S)\right| \ge \frac{\alpha}{\cP_{\cX}(S)}\quad\text{and}\quad \lambda = \textit{sign}(\bell(S) - \ell(S))
\]
\end{restatable}

Lemma \ref{lem:auditor} implies that when we find an empirical consistency violation using Algorithm \ref{alg:auditor-single-set}, it is indeed a real $\alpha$-consistency violation with respect to the true distribution, allowing us to make progress --- this guarantees that our algorithm will not run for too many iterations. But we need a converse condition, in order to make sure that we don't halt too early: we must show that if there are no empirical $\alpha'$-consistency violations for some $\alpha' > \alpha$, then there are also no $\alpha$-consistency violations with respect to the true distribution. This is what we do in Lemma \ref{lem:alpha-prime-auditor}. Observe, that without loss of generality, we can restrict attention to sets such that $\cP_{\cX}(S) \ge \alpha$ because any estimator in the range $[0,1]$ is trivially $\alpha$-mean consistent on every set with measure $< \alpha$.

\begin{restatable}{lemma}{auditorlemlower}
\label{lem:alpha-prime-auditor}
Fix any set $S \subseteq \cX$ such that $\cP_{\cX}(S) \geq \alpha$.  Assume $n$ is sufficiently large  such that $2\sqrt{\frac{\ln(\frac{2}{\delta})}{2n}} < \alpha$  If $D$ is approximately close to $\cP$ with respect to $(S, \bell, \ell)$, we have
\[
        \left|\bell(S) - \ell(S)\right| \ge \frac{\alpha'}{\cP_{\cX}(S)} \implies \text{Auditor}(\ell, \bell, \alpha, D_S) = \text{YES},
\]
    where $\alpha' = \alpha + 4\sqrt{\frac{1}{2n}\ln(\frac{2}{\delta})} + \left(\alpha -2\sqrt{\frac{\ln(\frac{2}{\delta})}{2n}}\right)^{-2}\left(2\sqrt{\frac{\ln(\frac{2}{\delta})}{2n}}\right)$.
\end{restatable}

The Auditor subroutine above performs a consistency check on a single set. We now use it to audit for mean consistency and pseudo-moment consistency across a collection of sets.

\begin{algorithm}[H]
\begin{algorithmic}
\FOR{$S \in \cS$}
    \IF{$|D_S| > 0$ and Auditor$(\ell, \bell, \alpha, \delta, D_S) = YES, \lambda$}
        \STATE return $S, \lambda$
    \ENDIF
\ENDFOR
\STATE return $NULL$
\end{algorithmic}
\caption{ConsistencyAuditor$(\bell, \ell, \alpha, \delta, D, \cS)$}
\label{alg:consistency-auditor}
\end{algorithm}

\begin{restatable}{corollary}{meanconsistencycor}
\label{cor:consistency-auditor}
Fix $\bell, \ell, \alpha, \delta$, and a collection of sets $\cS$. Given a set of $n$ points $D$ drawn i.i.d. from $\cP$ where $\alpha > 2 \sqrt{\frac{\ln(\frac{2}{\delta})}{2n}}$, ConsistencyAuditor$(\bmu, \alpha, D, \cS)$ has the following guarantee with probability $1-3\delta|\cS|$ over the randomness of $D$:
\begin{enumerate}
    \item If ConsistencyAuditor does output some set $S$ and $\lambda$, then \[  
        \left\vert \bell(S) - \ell(S) \right\vert \ge \frac{\alpha}{\cP_\cX(S)}
        \quad\text{and}\quad \lambda = \text{sign}(\bell(S) - \ell(S)).
    \]
    \item If ConsistencyAuditor outputs $NULL$, then for all $S \in \cS$,
    \[
        \left\vert \bell(S) - \ell(S) \right\vert \le \frac{\alpha'}{\cP_\cX(S)},
    \]
    where $\alpha' = \alpha + 4\sqrt{\frac{1}{2n}\ln(\frac{2}{\delta})} + \left(\alpha -2\sqrt{\frac{\ln(\frac{2}{\delta})}{2n}}\right)^{-2}\left(2\sqrt{\frac{\ln(\frac{2}{\delta})}{2n}}\right)$.
\end{enumerate}
\end{restatable}

 Thus, to detect a set $S$ with $\alpha$-mean consistency violation in line \ref{algline:wrapper} of Algorithm \ref{alg:wrapper}, we can leverage Algorithm \ref{alg:consistency-auditor} by drawing a fresh sample of size $n$ and setting $\bell(x) = \bmu(x), \ell(x,y)= y$, and $\cS =  \{G(\bmu, i): G \in \cG, i \in [m]\} \cup \{G(\bmu,\bma,i,j): G \in \cG, i,j \in [m], a \in \{2, \dots, k\}\}$. Likewise, for every $a \in \{2, \dots, k\}$, to detect a set $S$ with $\beta$-pseudo-moment consistency violation in line \ref{algline:pseudo-moment} of Algorithm \ref{alg:moment}, we can leverage Algorithm \ref{alg:consistency-auditor} by drawing a fresh sample of size $n$ and setting $\bell(x) = \bmk(x)$, $\ell(x,y) = (y-\bmu(x))^k$, and $\cS =  \{G(\bmu,\bma,i,j): G \in \cG, i,j \in [m]\}$. We write out the pseudocode of this process below.

 \begin{algorithm}[H]
\begin{algorithmic}
\STATE $D=\{(x_b, y_b)\}_{b=1}^n \sim \cP^n$ 
\STATE $\cS = \{G(\bmu,\bma,i,j): G \in \cG, i,j \in [m]\}$
\STATE $\bell(x) = \bma(x)$
\STATE $\ell(x,y) = (y- \bmu(x))^a$
\STATE $R,\psi = \text{ConsistencyAuditor}(\bell, \ell, \beta, \delta, D, \cS)$
    \WHILE{$R,\psi \neq NULL$}
            \STATE $\bma(x) = \begin{cases}
            	\text{project}_{[0,1]}(\bma(x)  - \beta \psi) & \text{if } x \in R \\
            	\bma(x) & \text{otherwise.}
            	\end{cases}$
            \STATE $D=\{(x_b, y_b)\}_{b=1}^n \sim \cP^n$ 
            \STATE $\cS = \{G(\bmu,\bma,i,j): G \in \cG, i,j \in [m]\}$
            \STATE $\bell(x) = \bma(x)$
            \STATE $R, \psi = \text{ConsistencyAuditor}(\bell,\ell, \beta, \delta, D, \cS)$
\ENDWHILE
\STATE return $\bmk$
\end{algorithmic}
\caption{$\text{PseudoMomentConsistencyFinite}(a, \beta, \delta, \bmu, \bma, n, \cG)$}
\label{alg:moment-finite}
\end{algorithm}
 
 \begin{algorithm}[H]
\begin{algorithmic}
\STATE initialize $\bmu^1(x) = 0$ for all $x$
    \STATE for all $1 < a \leq k$, initialize $\bma^1(x) = 0$ for all $x$
\STATE $t=1$
\STATE $\bell^t(x) = \bmu(x)$
\STATE $\ell(x,y) = y$
\STATE $D^t=\{(x_b, y_b)\}_{b=1}^n \sim \cP^n$ 
\STATE $\cS^t = \{G(\bmu^t, i): G \in \cG, i \in [m]\} \cup \{G(\bmu^t,\bma^t,i,j): G \in \cG, i,j \in [m], a \in \{2, \dots, k\}\}$
\STATE $S^t,\lambda^t = \text{ConsistencyAuditor}(\bell^t,\ell, \alpha, \delta, D, \cS)$
\WHILE{$S^t,\lambda^t \neq NULL$}
    \STATE $\bmu^{t+1} = \text{MeanConsistencyUpdate}( \bmu^{t}, S^t,\lambda^t)$
    \FOR{$a=2, \dots, k$}
        \STATE $\bma^{t+1} = \text{PseudoMomentConsistencyFinite}(a, \beta, \delta, \bmu^{t+1}, \bma^{t}, n, \cG)$.
    \ENDFOR
    \STATE $t = t+1$
    \STATE $\bell^t(x) = \bmu(x)$
    \STATE $D^t=\{(x_b, y_b)\}_{b=1}^n \sim \cP^n$ 
    \STATE $\cS^t = \{G(\bmu^t, i): G \in \cG, i \in [m]\} \cup \{G(\bmu^t,\bma^t,i,j): G \in \cG, i,j \in [m], a \in \{2, \dots, k\}\}$
    \STATE $S^t, \lambda^t = \text{ConsistencyAuditor}(\bell^t,\ell, \alpha, \delta, D^t, \cS^t)$
\ENDWHILE
\STATE return $(\bmu^t, \{\bma^t\}_{a=2}^k)$
\end{algorithmic}
\caption{$\text{AlternatingGradientDescentFinite}(\alpha,\beta, \delta, n, \cG)$}
\label{alg:wrapper-finite}
\end{algorithm}

\begin{restatable}{theorem}{finitewrappertheorem}
\label{thm:finite-wrapper}
Let $T$ be the final iterate of Algorithm \ref{alg:wrapper-finite}. If $2\sqrt{\frac{\ln(\frac{2}{\delta})}{2n}} \le \alpha$ and $2\sqrt{\frac{\ln(\frac{2}{\delta})}{2n}} \le \beta$, we have the following guarantees:
\begin{enumerate}
    \item \textbf{Total Iterations}: With probability $1-3\delta|\cG|Q_\alpha\left((m^2 + m)+ m^2Q_\beta\right)$ over the randomness of our samples, the final iterate $T$ is s.t. $T \le \frac{1}{\alpha^2}-1$ and the total number of gradient descent update operations will be at most $Q$, where
    \begin{align*}
        Q_\alpha = \frac{1}{\alpha^2}-1,
        Q_\beta = (k-1)\left(\frac{1}{\beta^2}-1\right),
        Q = Q_\alpha (1 + Q_\beta).
    \end{align*}
    In particular, the algorithm uses at most $n Q$ samples from $\cP$. 
    
    \item \textbf{Mean multicalibration}: With probability $1-3\delta (m^2+m)|\cG|$, output $\bmu^T$ is $\alpha'$-mean multicalibrated with respect to $\cG$ where \[\alpha' = \alpha + 4\sqrt{\frac{\ln\left(\frac{2}{\delta}\right)}{2n}} + \left(\alpha -2\sqrt{\frac{\ln\left(\frac{2}{\delta}\right)}{2n}}\right)^{-2}\left(2\sqrt{\frac{\ln(\frac{2}{\delta})}{2n}}\right).\]
    
    \item \textbf{Mean Conditioned Moment multicalibration}: With probability $1-3\delta|\cG|( km^2  +m)$, for any $a \in \{2, \dots, k\}$, pair $(\bmu^T, \bma^T)$ is $(\alpha', a\alpha' + \beta', {a \over m})$-mean-conditioned-moment calibrated where
    \[\alpha' = \alpha + 4\sqrt{\frac{\ln\left(\frac{2}{\delta}\right)}{2n}} + \left(\alpha -2\sqrt{\frac{\ln\left(\frac{2}{\delta}\right)}{2n}}\right)^{-2}\left(2\sqrt{\frac{\ln(\frac{2}{\delta})}{2n}}\right)\]
    \[
        \beta'= \beta + 4\sqrt{\frac{\ln\left(\frac{2}{\delta}\right)}{2n}} + \left(\beta -2\sqrt{\frac{\ln\left(\frac{2}{\delta}\right)}{2n}}\right)^{-2}\left(2\sqrt{\frac{\ln(\frac{2}{\delta})}{2n}}\right)
    \]
\end{enumerate}
\end{restatable}

The following corollary derives the sample complexity of our algorithm, implied by Theorem \ref{thm:finite-wrapper}, for a target set of parameters. Observe that the sample complexity is polynomial in $k,m,1/\alpha,1/\beta,1/\epsilon, \log(1/\delta),$ and $\log|\cG|$.

\begin{restatable}{corollary}{finitewrappercor}
\label{cor:finite-wrapper}
 Fix target parameters $\alpha', \beta', \delta'$ and $\epsilon > 0$ such that $\epsilon < \alpha'$ and $\epsilon < \beta'$. 
 Define 
 \[
    \bQ = \frac{6 |\cG|km^2}{\left(\frac{\alpha' - \epsilon}{6+\frac{2}{\epsilon^2}}\right)^2\left(\frac{\beta' - \epsilon}{6+\frac{2}{\epsilon^2}}\right)^2}
    ,\quad
    \delta = \frac{\delta'}{\max(3|\cG|(km^2 + m), \bQ)},\quad
    n_\alpha =  \frac{\ln(\frac{2\bQ}{\delta})}{2\left(\frac{\alpha' - \epsilon}{6+\frac{2}{\epsilon^2}}\right)^2},\quad
    n_\alpha =  \frac{\ln(\frac{2\bQ}{\delta})}{2\left(\frac{\beta' - \epsilon}{6+\frac{2}{\epsilon^2}}\right)^2},\quad.
 \]
 
 Then, AlternatingGradientDescentFinite$(\alpha, \beta, \delta, n, \cG)$ where 
 \begin{align*}
 &\alpha = 2\sqrt{\frac{\ln(\frac{2\bQ}{\delta})}{2n_\alpha}} + \epsilon, \beta = 2\sqrt{\frac{\ln(\frac{2\bQ}{\delta})}{2n_\beta}} + \epsilon, \\
  &n=\max\left(\frac{\ln(\frac{2\bQ}{\delta})}{\ln(\frac{2}{\delta})} n_\alpha, \frac{\ln(\frac{2\bQ}{\delta})}{\ln(\frac{2}{\delta})} n_\beta, \frac{2\ln(\frac{2}{\delta})}{\alpha^2}, \frac{2\ln(\frac{2}{\delta})}{\beta^2}\right)
 \end{align*}
 has the following guarantees with probability $1-\delta'$:
 \begin{enumerate}
     \item The total number of gradient descent updates will be at most $Q$, where $Q$ is as defined in Theorem \ref{thm:finite-wrapper}.
     \item $\bmu^T$ is $\alpha'$-mean-calibrated.
     \item For every $a \in \{2, \dots, k\}$, $(\bmu^T, \bma^T)$ is $(\alpha', a\alpha' + \beta', \frac{a}{m})$-mean-conditioned-moment calibrated.
 \end{enumerate}
\end{restatable}

Finally, we state the running time of the algorithm in the following Theorem.
\begin{restatable}{theorem}{runtimethm}
With probability $1-3\delta|\cG|Q_\alpha\left((m^2 + m)+ m^2Q_\beta\right)$, the running time of Algorithm \ref{alg:wrapper-finite} is $O\left(Q |\cG| m^2 n\right) = O\left(\frac{k|\cG|m^2n}{\alpha^2\beta^2}\right)$ where $Q_\alpha, Q_\beta$, and $Q$ are as defined in Theorem \ref{thm:finite-wrapper}.
\end{restatable}

\subsection{Oracle Efficient Implementation}
In Section \ref{enumeration} we analyzed an algorithm that had favorable sample-complexity bounds, but was computationally expensive when $\cG$ was large: although it ran for only a small number of iterations, each iteration required a complete enumeration of every set in $\cG$. In this section, we show how to replace this expensive step with a call to an algorithm which can solve learning problems over $\cG$, if one is available. Because the remaining portion of the algorithm is computationally efficient --- even if $\cG$ is very large --- this yields what is sometimes known as an ``oracle efficient algorithm''. Similar reductions have been  given in \citet{multicalibration, gerrymandering,multiaccuracy}.

\begin{definition}
For some $\rho \in [0,1]$ and non-increasing function $p: \mathbb{N} \to [0,1]$, $A$ is a $(\rho, p)$-agnostic learning oracle for hypothesis class $\cH \subseteq 2^{\cX}$ with respect to a label function $r(x,y) \in [-1,1]$, if for any distribution $\cP$, given $n$ random samples from $\cP$, it outputs $f: \cX \to \{0,1\}$ such that with probability $1-p(n)$, 
\[
\E_{(x,y) \sim \cP}[f(x) \cdot r(x,y)] + \rho \ge \sup_{h \in \cH} \E_{(x,y) \sim \cP}[h(x) \cdot r(x,y)].
\]

We write $\tau(n)$ to denote the running time of the oracle $A$ when $n$ data points are used, which we assume is at least $\Omega(n)$.
\end{definition}
\begin{remark}
A more common definition of an agnostic learning oracle would use hypotheses with range $\{-1,1\}$ rather than $\{0,1\}$. But this definition will be more convenient for us, and is equivalent (up to a constant factor in the parameters) via a linear transformation. 
\end{remark}

We will use a learning algorithm for any class $\cH$ such that $\cG \subseteq \cH$ to replace the set enumeration steps of our algorithm. In particular, to find a set of the form $G(\bmu,i)$ on which our existing predictor $\bmu$ fails to be mean consistent, we run our learning algorithm on the subset of our sample that intersects with $\cX(\bmu,i)$, labelled with the positive and negative \emph{residuals} of our predictor --- i.e. on the labels $r^+_R(x,y) = \bmu(x)-y$ and $r^-_R = y - \bmu(x)$. Similarly, to find a set of the form $G(\bmu,\bma,i,j)$, we run our learning algorithm on the sets $\cX(\bmu,\bma,i,j)$ labeled with both the positive and negative residuals. Finding sets on which we fail to be moment pseudo-consistent with respect to $\bmu$ is similar. All in all, this requires $O(k\cdot m^2)$ runs of our learning algorithm per gradient descent step, replacing the complete enumeration of the collection of sets $\cG$.  We make this process more precise in Algorithm \ref{alg:consistency-auditor-learning-oracle-wrapper} and state the guarantees in Theorem \ref{thm:oracle-auditor}. We also include the pseudocode for the correspondingly updated AlternatingGradientDescentFinite using Algorithm \ref{alg:consistency-auditor-learning-oracle-wrapper} as the auditing subroutine in the appendix -- see Algorithm \ref{alg:wrapper-finite-oracle}.

\begin{algorithm}[H]
\begin{algorithmic}
\STATE $r_R^+(x,y) = \begin{cases}\bell(x) - \ell(x,y) &\text{ if $x \in R$}\\ 0 &\text{ otherwise}\end{cases}\quad\text{and}\quad D_R^+ = \{(x_b, r_R^+(x,y))\}_{b=1}^{n}$
\STATE $r_R^-(x,y) = \begin{cases}\ell(x,y) - \bell(x) &\text{ if $x \in R$}\\ 0 &\text{ otherwise}\end{cases} \quad\text{and}\quad D_R^- = \{(x_b, r_R^-(x,y))\}_{b=1}^{n}$
\STATE $\chi_{S^+} = A(D_R^+, \cH)$ 
\STATE $\chi_{S^-} = A(D_R^-, \cH)$
\STATE return $(S^+, S^-)$
\end{algorithmic}
\caption{LearningOracleConsistencyAuditor$(\bell, \ell, \alpha, \delta, D, R, A)$}
\label{alg:consistency-auditor-learning-oracle}
\end{algorithm}

\newcommand{\cV}{\mathcal{V}}

\begin{algorithm}[H]
\begin{algorithmic}
\STATE $\cV =\{\}$
\FOR{$R \in \cR$}
    \IF{$|D_R| > 0$}
        \STATE $S^+, S^- = \text{LearningOracleConsistencyAuditor}(\bell, \ell, \alpha, \delta, D, D^{check}, R, A)$
        \STATE $\cV = \cV \cup \{(S^+ \cap R), (S^- \cap R)\}$
    \ENDIF
\ENDFOR
\STATE return ConsistencyAuditor$(\bell, \ell, \alpha, D^{\text{check}}, \cV)$
\end{algorithmic}
\caption{LearningOracleConsistencyAuditorWrapper$(\bell, \ell, \alpha, \delta, D, D^{\text{check}}, \cR, A)$}
\label{alg:consistency-auditor-learning-oracle-wrapper}
\end{algorithm}

First we observe that the objective of the agnostic learning oracle on the sets we run it on corresponds directly to the (positive and negative) violation of mean consistency on these sets, weighted by the measure of the sets.

\begin{restatable}{lemma}{oraclelem}
\label{lem:oracle-guarantee}
For each $R \in \cR$ and any $\chi_S$:
\[
\E_{(x,y)}[\chi_S(x) \cdot r_R^+(x,y)] = \cP_\cX(R \cap S) \left(\bell(R \cap S) - \ell(R \cap S)\right)\]
\[\E_{(x,y)}[\chi_S(x) \cdot r_R^-(x,y)] = \cP_\cX(R \cap S) \left(\ell(R \cap S) - \bell(R \cap S)\right)
\]
\end{restatable}

Using this, we can show that our learning oracle based consistency auditor has comparable guarantees to the consistency auditor that operated via set enumeration:

\begin{restatable}{theorem}{oracleauditor}
\label{thm:oracle-auditor}
Assume $n$ is sufficiently large such that $\alpha > 2\sqrt{\frac{\ln(\frac{2}{\delta})}{2n}}$. Algorithm \ref{alg:consistency-auditor-learning-oracle-wrapper} has the following guarantees:
\begin{enumerate}
    \item If it returns some $S$ and $\lambda$, then with probability $1-3\delta|\cR|$ over the randomness of $D^{\text{check}}$, 
    \[
        |\bell(S) -\ell(S)| \ge \frac{\alpha}{\cP_\cX(S)}. 
    \]
    \item If it returns $NULL$, then with probability $1-|\cR|(3\delta+2p(n))$ over the randomness of $D$ and $D^{\text{check}}$, for all $\chi_S \in \cH$ and $R \in \cR$,
    \[
        \vert \bell(R \cap S) - \ell(R \cap S)\vert \le \frac{\alpha' + \rho}{\cP_\cX(R \cap S)},
    \]
    where $\alpha'$ is as defined in Corollary \ref{cor:consistency-auditor}.
\end{enumerate}
\end{restatable}

Observe that when $\cG \subseteq \cH$ and $\cR= \{\cX(\bmu, \bma, i, j): i,j \in [m]\}$, then, the collection of intersections $R \cap S$ over all $\chi_S \in \cH$ and $R \in \cR$ contains $\{G(\bmu, \bma, i, j): G \in \cG, i,j \in [m]\}$. The same observation applies when $\cR = \{\cX(\bmu, i): i \in [m]\} \cup \{\cX(\bmu,\bma,i, j): i,j \in [m]\}$ -- the collection of intersections includes $\{G(\bmu, i): G \in \cG, i \in [m]\} \cup \{\cX(\bmu,\bma,i, j): G \in \cG, i,j \in [m]\}$.

We now present the guarantees of a version of AlternatingGradientDescent that uses Algorithm \ref{alg:consistency-auditor-learning-oracle-wrapper} as the auditor. Its pseudo-code can be found as Algorithm  \ref{alg:wrapper-finite-oracle} in the appendix. We elide the proof as it is almost identical to that of Theorem \ref{thm:wrapper-distr-known}
 and Theorem \ref{thm:finite-wrapper}.

\begin{theorem}
 \label{thm:agnostic}
 Assume $\cG \subseteq \cH$. Let $T$ be the final iterate of Algorithm \ref{alg:wrapper-finite-oracle}. If $2\sqrt{\frac{\ln(\frac{2}{\delta})}{2n}} \le \alpha$, $2\sqrt{\frac{\ln(\frac{2}{\delta})}{2n}} \le \beta$, and $\cG \subseteq \cH$, we have the following guarantees:
\begin{enumerate}
    \item \textbf{Total Iterations}: With probability $1-3\delta Q_\alpha\left((m^2 + m)+ m^2Q_\beta\right)$ over the randomness of our samples, the final iterate $T$ is such that $T \le \frac{1}{\alpha^2}-1$ and the total number of gradient descent update operations will be at most $Q$, where $Q_\alpha, Q_\beta$, and $Q$ are all as defined in Theorem \ref{thm:finite-wrapper}.
    
    In particular, the algorithm uses at most $O(n Q)$ samples from $\cP$. 
    
    \item \textbf{Mean multicalibration}: With probability $1-(m^2+m)(3\delta+2p(n))$, output $\bmu^T$ is $\alpha''$-mean multicalibrated with respect to $\cG$ where $\alpha'' = \alpha' + \rho$ and $\alpha'$ is as defined in Theorem \ref{thm:finite-wrapper}.
    
    \item \textbf{Mean Conditioned Moment multicalibration}: With probability $1-( km^2  +m)(3\delta + 2p(n))$, for any $a \in \{2, \dots, k\}$, pair $(\bmu^T, \bma^T)$ is $(\alpha'', a\alpha'' + \beta'', {a \over m})$-mean-conditioned-moment calibrated where
    $\beta'' = \beta' + \rho$ and $\beta'$ is as defined in Theorem \ref{thm:finite-wrapper}.
\end{enumerate}
\end{theorem}

Finally, we state the running time of Algorithm \ref{alg:wrapper-finite-oracle}.
\begin{restatable}{theorem}{oracleruntime}
\label{thm:wrapper-oracle-running-time}
 With probability at least $1-3\delta Q_\alpha\left((m^2 + m)+ m^2Q_\beta\right)$, the running time of Algorithm \ref{alg:wrapper-finite-oracle} is bounded by  $O(Qm^2\tau(n))$, where $Q$ is the total number of gradient descent operations as defined in Theorem \ref{thm:finite-wrapper}.
\end{restatable}

\section{Marginal Prediction Intervals}\label{sec:intervals}
We now present an application of our results. Given subgroups of interest $\cG$, we have shown how to to construct a multicalibrated mean predictor $\bmu$ and  moment predictors $(\bma)_{a=2}^k$ that are simultaneously mean-conditioned moment-multicalibrated. A key question is whether we can use mean-conditioned moment multicalibrated predictors in applications in which we would use real distributional moments, were they available. 

In this section, we show that the answer is \emph{yes} in an important application. Mean-conditioned moment multicalibrated  predictors can be used in tail bounds just as real moments could be to compute prediction intervals. Where real moments would yield prediction intervals conditioned on an individual vector of features $x$, mean-conditioned moment-multicalibrated predictors when used in the same computations yield  marginal prediction intervals that are simultaneously valid for every sufficiently large subgroup. In particular, given a coverage failure probability $\delta$ and a group size $\gamma$ we show how to construct just from mean and moment predictions, for every $x \in X$, an interval $I(x,\gamma)$ such that for every $G \in \cG$ and for every pair of predictions $i,j$ such that $G(\bmu,\bma,i,j)$ has mass at least $\gamma$ we have: $\Pr_{(x,y)}[y\in I(x,\gamma)| x\in G(\bmu,\bma,i,j)] \geq 1-\delta.$

Recall the following tail inequality (a simple consequence of Markov's inequality: when $k = 2$, it is known as Chebyshev's inequality):
\begin{lemma}\label{lem:cheby}
Let $X$ be a random variable with mean $\mu$. Then for even $k$, $t > 0$:
\begin{align*}
    &\Pr[|X - \mu| \geq t] \leq \frac{\E[(X-\mu)^k] }{t^k}.
\end{align*}
\end{lemma}

Suppose we knew the \emph{real} moments $\mk(x)$ of the distribution on $y$ conditional on features $x$: A direct application of the above lemma would allow us to conclude that for any even moment $k$: 
$$\Pr\left[y \not \in \left[\mu(x) - \left(\frac{\mk(x)}{\delta}\right)^{\frac1k}, \mu(x) + \left(\frac{\mk(x)}{\delta}\right)^{\frac1k}\right] \middle\vert x \right]\leq \delta.$$

Bounds of this form are simple, but also strong: there is always an integer moment $k$ such that the above bound is at least as tight as a generalized Chernoff bound\footnote{Chernoff's bound is $\Pr[X \geq t] \leq \inf_{\theta \geq 0} M_X(\theta)e^{-\theta t}$, where $M_X(\theta)$ is the moment generating function for $X$} \citep{philips1995moment}, and only the first $k \leq O(\log(1/\delta))$ moments are necessary to match Chernoff bounds at coverage probability $1-\delta$ \citep{schmidt1995chernoff}.

If we had \emph{exactly} mean-conditioned moment-calibrated predictors $(\bmu,\bmk)$ for some $k$ even, over a set of groups $G$, we would obtain \emph{exactly} the same bound using these predictors as a marginal prediction interval: i.e. we would obtain for every $G \in \cG$,  and every $i,j$:
$$\Pr_{(x,y)}\left[y \not \in \left[\bmu(x) - \left(\frac{\bmk(x)}{\delta}\right)^{1/k}, \bmu(x) + \left(\frac{\bmk(x)}{\delta}\right)^{1/k}\right] \middle\vert x \in G(\bmu,\bmk,i,j) \right]\leq \delta.$$

This is because mean-conditioned moment-multicalibrated predictors actually do provide real distributional moments, over the selection of a random point within any set $G(\bmu,\bmk,i,j)$. Of course we only have approximately mean-conditioned-moment multicalibrated predictors. Given $(\alpha, \beta,\epsilon)$-mean-conditioned-moment multicalibrated predictors $(\bmu, \bmk)$, $k$ even, with respect to some collection of groups $\cG$, we can endow our predictions with (marginal) prediction intervals that have coverage probability $1-\delta$ as follows. The \emph{width} of our prediction interval for a point $x$ will be:
\begin{align*}
\Delta_{\gamma,k}(x) &= \frac{\alpha}{\gamma} + \epsilon + \frac{1}{m} + \left(\frac{\bmk(x) + \epsilon + \frac{1}{m} +\frac{\beta}{\gamma}}{\delta}\right)^{\frac{1}{k}}, 
\end{align*}
Our  prediction interval for  $x$ will be centered at its predicted mean, and is defined as follows:
$$I_{\gamma,k}(x) =  [\bmu(x)-\Delta_{\gamma,k}(x), \bmu(x)+ \Delta_{\gamma,k}(x)].$$ 

These are valid marginal prediction intervals as averaged over \emph{any} set of the form $G(\bmu,\bmk,i,j)$ that has measure larger than $\gamma$. Note that all of the approximation terms $\alpha, \beta, \epsilon, 1/m$ are terms that we can drive to zero at polynomial cost in running time and sample complexity. 

\begin{theorem}\label{thm:ciguarantee}
 Assume that $\bmu, \bmk$ is $(\alpha, \beta,\epsilon)$-mean-conditioned moment multicalibrated with respect to $\cG$, with $k$ even. Then for any group $G \in \cG$ and any set $G(\bmu,\bmk, i,j)$ such that $\cP_\cX[G(\bmu,\bmk, i,j)] \geq \gamma$, we have:
\begin{align*}
    \cP[y \not\in I_{\gamma,k}(x)| x\in G(\bmu,\bmk, i,j)] \leq \delta 
\end{align*} 
\end{theorem}
\begin{proof}
To see this note that:
\begin{align*}
&\cP[y \not\in  I_{\gamma,k}(x)| x\in G(\bmu,\bmk, i,j)] \\
=& \cP \left[ \left\vert y - \bmu(x) \right\vert \ge\frac{\alpha}{\gamma} + \frac{1}{m} + \epsilon + \left(\frac{\bmk(x) + \frac{1}{m} + \epsilon + \frac{\beta}{\gamma}}{\delta}\right)^{\frac{1}{k}} \middle\vert x \in G(\bmu,\bmk,i,j) \right]\\
\leq& \cP\left[ \left\vert y - \bmu(G(\bmu, \bmk, i,j)) \right\vert \ge \frac{\alpha}{\gamma}+\epsilon + \left(\frac{\bmk(G(\bmu,\bmk,i,j))+\frac{\beta}{\gamma} + \epsilon}{\delta}\right)^{\frac{1}{k}}  \middle\vert x \in G(\bmu, \bmk, i,j)\right]\\
\leq& \cP\left[ \left\vert y - \mu(G(\bmu, \bmk, i,j)) \right\vert \ge \left(\frac{\mk(G(\bmu,\bmk,i,j))}{\delta}\right)^{\frac{1}{k}}  \middle\vert x \in G(\bmu, \bmk, i,j)\right]\\
\le&   \delta
\end{align*}

Here, the first inequality follows from the fact that all $x \in G(\bmu,\bmk,i,j)$ are (by definition) such that $|\bmu(x) - \frac{i}{m}| \leq \frac{1}{2m}$ and $|\bmk(x) - \frac{j}{m}| \leq \frac{1}{2m}$, and hence $|\bmu(x) -  \bmu(G(\bmu, \bmk, i,j))| \leq \frac{1}{m}$ and $|\bmk(x) - \bmk(G(\bmu, \bmk, i,j)))| \leq \frac{1}{m}$. The second inequality follows from the definition of $(\alpha,\beta)$-mean conditioned moment multicalibration and the fact that $\cP[G(\bmu,\bmk, i,j)] \geq \gamma$. Finally, once we have replaced our mean and moment estimates with the true mean and moment of $G(\bmu,\bmk, i,j)$, the final inequality follows as an application of Lemma \ref{lem:cheby}.

\end{proof}

This theorem shows how---given $(\alpha,\beta, \epsilon)$ mean-conditioned moment-multicalibrated predictors $\bmu,\bmk$---we can construct prediction intervals for any set $G(\bmu, \bmk, i,j)$ with probability larger than $\gamma$.\footnote{We showed this just for \emph{even} moments $k$ --- but a version of Lemma \ref{lem:cheby} also holds for $k$ odd, i.e. for any r.v. $X$ with mean $\mu$, any number $k$, and any $t>0$, we have $\Pr[|X-\mu|\geq t] \leq \frac{\mathbb{E}[|(X-\mu)^k]}{t^k}$. We can use this to construct valid confidence intervals using ``absolute central moments'' of any degree, even or odd. Note also that our algorithms and analysis apply  identically if the goal was to provide mean-conditioned, multicalibrated predictors of absolute central moments (i.e. the analog of Definition \ref{def:calibrated} but where instead of $\mk(\cdot)$, we calibrate our predictor to the analogous absolute central moment). } However, we have more information available to us: We have mean conditioned moment-calibrated predictors for all moments 2 thru $k$, $(\bma)_{a=2}^k.$ A straightforward valid solution is to pick some even moment $a$ s.t. $1<a\leq k$, and then construct prediction intervals as above. We could optimize our choice of $a$ so as to minimize e.g. the expected width of the prediction intervals over a random choice of $x$.  But this leads to the question of whether we can do better by using more than one moment estimator at a time. In Appendix \ref{app:formulation} we show that this problem reduces to the venerable submodular-cost set-cover problem. Known approximation guarantees for this problem are relatively weak in this context (scaling with $\log|\cX|$, which will typically be linear in the \emph{dimension} of the data). We leave the question of how to  optimally use multiple mean-conditioned moment multicalibrated predictors---taking advantage of multiple moments simultaneously---to future research. 

\subsection*{Acknowledgements}
We are thankful for helpful early conversations with Sampath Kannan. We gratefully acknowledge support from NSF grants CCF-1763307 and CCF-1763349 (Jung, Pai, Roth, and Vohra), and  NSF grant CCF-1934876 and an Amazon Research Award (Roth).

\bibliographystyle{plainnat}
\bibliography{refs}

\begin{thebibliography}{30}
\providecommand{\natexlab}[1]{#1}
\providecommand{\url}[1]{\texttt{#1}}
\expandafter\ifx\csname urlstyle\endcsname\relax
  \providecommand{\doi}[1]{doi: #1}\else
  \providecommand{\doi}{doi: \begingroup \urlstyle{rm}\Url}\fi

\bibitem[Agarwal et~al.(2018)Agarwal, Beygelzimer, Dud{\'\i}k, Langford, and
  Wallach]{reductions}
Alekh Agarwal, Alina Beygelzimer, Miroslav Dud{\'\i}k, John Langford, and
  Hanna~M Wallach.
\newblock A reductions approach to fair classification.
\newblock In \emph{ICML}, 2018.

\bibitem[Barber et~al.(2019)Barber, Candes, Ramdas, and
  Tibshirani]{barber2019limits}
Rina~Foygel Barber, Emmanuel~J Candes, Aaditya Ramdas, and Ryan~J Tibshirani.
\newblock The limits of distribution-free conditional predictive inference.
\newblock \emph{arXiv preprint arXiv:1903.04684}, 2019.

\bibitem[Bassily et~al.(2016)Bassily, Nissim, Smith, Steinke, Stemmer, and
  Ullman]{ada2}
Raef Bassily, Kobbi Nissim, Adam Smith, Thomas Steinke, Uri Stemmer, and
  Jonathan Ullman.
\newblock Algorithmic stability for adaptive data analysis.
\newblock In \emph{Proceedings of the forty-eighth annual ACM symposium on
  Theory of Computing}, pages 1046--1059, 2016.

\bibitem[Chouldechova and Roth(2020)]{fairsurvey}
Alexandra Chouldechova and Aaron Roth.
\newblock A snapshot of the frontiers of fairness in machine learning.
\newblock \emph{Communications of the ACM}, 63\penalty0 (5):\penalty0 82--89,
  2020.

\bibitem[Dawid(1982)]{dawid1982well}
A~Philip Dawid.
\newblock The well-calibrated bayesian.
\newblock \emph{Journal of the American Statistical Association}, 77\penalty0
  (379):\penalty0 605--610, 1982.

\bibitem[Dwork et~al.(2012)Dwork, Hardt, Pitassi, Reingold, and
  Zemel]{awareness}
Cynthia Dwork, Moritz Hardt, Toniann Pitassi, Omer Reingold, and Richard Zemel.
\newblock Fairness through awareness.
\newblock In \emph{Proceedings of the 3rd innovations in theoretical computer
  science conference}, pages 214--226, 2012.

\bibitem[Dwork et~al.(2015)Dwork, Feldman, Hardt, Pitassi, Reingold, and
  Roth]{ada1}
Cynthia Dwork, Vitaly Feldman, Moritz Hardt, Toniann Pitassi, Omer Reingold,
  and Aaron~Leon Roth.
\newblock Preserving statistical validity in adaptive data analysis.
\newblock In \emph{Proceedings of the forty-seventh annual ACM symposium on
  Theory of computing}, pages 117--126, 2015.

\bibitem[Dwork et~al.(2019)Dwork, Kim, Reingold, Rothblum, and
  Yona]{dwork2019learning}
Cynthia Dwork, Michael~P Kim, Omer Reingold, Guy~N Rothblum, and Gal Yona.
\newblock Learning from outcomes: Evidence-based rankings.
\newblock In \emph{2019 IEEE 60th Annual Symposium on Foundations of Computer
  Science (FOCS)}, pages 106--125. IEEE, 2019.

\bibitem[Foster and Vohra(1998)]{FV98}
Dean~P Foster and Rakesh~V Vohra.
\newblock Asymptotic calibration.
\newblock \emph{Biometrika}, 85\penalty0 (2):\penalty0 379--390, 1998.

\bibitem[Hardt and Rothblum(2010)]{PMW}
Moritz Hardt and Guy~N Rothblum.
\newblock A multiplicative weights mechanism for privacy-preserving data
  analysis.
\newblock In \emph{2010 IEEE 51st Annual Symposium on Foundations of Computer
  Science}, pages 61--70. IEEE, 2010.

\bibitem[H{\'e}bert-Johnson et~al.(2018)H{\'e}bert-Johnson, Kim, Reingold, and
  Rothblum]{multicalibration}
{\'U}rsula H{\'e}bert-Johnson, Michael Kim, Omer Reingold, and Guy Rothblum.
\newblock Multicalibration: Calibration for the (computationally-identifiable)
  masses.
\newblock In \emph{International Conference on Machine Learning}, pages
  1939--1948, 2018.

\bibitem[Joseph et~al.(2016)Joseph, Kearns, Morgenstern, and Roth]{fairbandits}
Matthew Joseph, Michael Kearns, Jamie~H Morgenstern, and Aaron Roth.
\newblock Fairness in learning: Classic and contextual bandits.
\newblock In \emph{Advances in Neural Information Processing Systems}, pages
  325--333, 2016.

\bibitem[Joseph et~al.(2018)Joseph, Kearns, Morgenstern, Neel, and
  Roth]{fairbandits2}
Matthew Joseph, Michael Kearns, Jamie Morgenstern, Seth Neel, and Aaron Roth.
\newblock Meritocratic fairness for infinite and contextual bandits.
\newblock In \emph{Proceedings of the 2018 AAAI/ACM Conference on AI, Ethics,
  and Society}, pages 158--163, 2018.

\bibitem[Jung et~al.(2020)Jung, Ligett, Neel, Roth, Sharifi-Malvajerdi, and
  Shenfeld]{ada3}
Christopher Jung, Katrina Ligett, Seth Neel, Aaron Roth, Saeed
  Sharifi-Malvajerdi, and Moshe Shenfeld.
\newblock A new analysis of differential privacy’s generalization guarantees.
\newblock In \emph{11th Innovations in Theoretical Computer Science Conference
  (ITCS 2020)}. Schloss Dagstuhl-Leibniz-Zentrum f{\"u}r Informatik, 2020.

\bibitem[Kearns et~al.(2018)Kearns, Neel, Roth, and Wu]{gerrymandering}
Michael Kearns, Seth Neel, Aaron Roth, and Zhiwei~Steven Wu.
\newblock Preventing fairness gerrymandering: Auditing and learning for
  subgroup fairness.
\newblock In \emph{International Conference on Machine Learning}, pages
  2564--2572, 2018.

\bibitem[Kearns et~al.(2019)Kearns, Neel, Roth, and Wu]{subgroup}
Michael Kearns, Seth Neel, Aaron Roth, and Zhiwei~Steven Wu.
\newblock An empirical study of rich subgroup fairness for machine learning.
\newblock In \emph{Proceedings of the Conference on Fairness, Accountability,
  and Transparency}, pages 100--109, 2019.

\bibitem[Kim et~al.(2018)Kim, Reingold, and Rothblum]{cbawareness}
Michael Kim, Omer Reingold, and Guy Rothblum.
\newblock Fairness through computationally-bounded awareness.
\newblock In \emph{Advances in Neural Information Processing Systems}, pages
  4842--4852, 2018.

\bibitem[Kim et~al.(2019)Kim, Ghorbani, and Zou]{multiaccuracy}
Michael~P Kim, Amirata Ghorbani, and James Zou.
\newblock Multiaccuracy: Black-box post-processing for fairness in
  classification.
\newblock In \emph{Proceedings of the 2019 AAAI/ACM Conference on AI, Ethics,
  and Society}, pages 247--254, 2019.

\bibitem[Lehrer(2001)]{lehrer2001any}
Ehud Lehrer.
\newblock Any inspection is manipulable.
\newblock \emph{Econometrica}, 69\penalty0 (5):\penalty0 1333--1347, 2001.

\bibitem[Oakes(1985)]{oakes1985self}
David Oakes.
\newblock Self-calibrating priors do not exist.
\newblock \emph{Journal of the American Statistical Association}, 80\penalty0
  (390):\penalty0 339--339, 1985.

\bibitem[Philips and Nelson(1995)]{philips1995moment}
Thomas~K Philips and Randolph Nelson.
\newblock The moment bound is tighter than chernoff's bound for positive tail
  probabilities.
\newblock \emph{The American Statistician}, 49\penalty0 (2):\penalty0 175--178,
  1995.

\bibitem[Rothblum and Yona(2018)]{RY18}
Guy Rothblum and Gal Yona.
\newblock Probably approximately metric-fair learning.
\newblock In \emph{International Conference on Machine Learning}, pages
  5680--5688, 2018.

\bibitem[Sandroni et~al.(2003)Sandroni, Smorodinsky, and
  Vohra]{sandroni2003calibration}
Alvaro Sandroni, Rann Smorodinsky, and Rakesh~V Vohra.
\newblock Calibration with many checking rules.
\newblock \emph{Mathematics of operations Research}, 28\penalty0 (1):\penalty0
  141--153, 2003.

\bibitem[Schmidt et~al.(1995)Schmidt, Siegel, and
  Srinivasan]{schmidt1995chernoff}
Jeanette~P Schmidt, Alan Siegel, and Aravind Srinivasan.
\newblock Chernoff--hoeffding bounds for applications with limited
  independence.
\newblock \emph{SIAM Journal on Discrete Mathematics}, 8\penalty0 (2):\penalty0
  223--250, 1995.

\bibitem[Shabat et~al.(2020)Shabat, Cohen, and Mansour]{multiUC}
Eliran Shabat, Lee Cohen, and Yishay Mansour.
\newblock Sample complexity of uniform convergence for multicalibration.
\newblock \emph{arXiv preprint arXiv:2005.01757}, 2020.

\bibitem[Shafer and Vovk(2008)]{conformal}
Glenn Shafer and Vladimir Vovk.
\newblock A tutorial on conformal prediction.
\newblock \emph{Journal of Machine Learning Research}, 9\penalty0
  (Mar):\penalty0 371--421, 2008.

\bibitem[Sharifi-Malvajerdi et~al.(2019)Sharifi-Malvajerdi, Kearns, and
  Roth]{aif}
Saeed Sharifi-Malvajerdi, Michael Kearns, and Aaron Roth.
\newblock Average individual fairness: Algorithms, generalization and
  experiments.
\newblock In \emph{Advances in Neural Information Processing Systems}, pages
  8242--8251, 2019.

\bibitem[Wan et~al.(2010)Wan, Du, Pardalos, and Wu]{wan2010greedy}
Peng-Jun Wan, Ding-Zhu Du, Panos Pardalos, and Weili Wu.
\newblock Greedy approximations for minimum submodular cover with submodular
  cost.
\newblock \emph{Computational Optimization and Applications}, 45\penalty0
  (2):\penalty0 463--474, 2010.

\bibitem[Zhao et~al.(2020)Zhao, Ma, and Ermon]{individualcalibration}
Shengjia Zhao, Tengyu Ma, and Stefano Ermon.
\newblock Individual calibration with randomized forecasting.
\newblock \emph{arXiv preprint arXiv:2006.10288}, 2020.

\bibitem[Zinkevich(2003)]{Zin03}
Martin Zinkevich.
\newblock Online convex programming and generalized infinitesimal gradient
  ascent.
\newblock In \emph{Proceedings of the 20th international conference on machine
  learning (icml-03)}, pages 928--936, 2003.

\end{thebibliography}

\appendix
\section{Details and Proofs from Section \ref{subsec:mean-calibration}}

\lemregret*
\begin{proof}
Because $\bmu^{t+1}(x)= \bmu^{t}(x)$ for $x \not \in S^t$, we can lower bound the ``progress'' made towards $\mu$ at each round $t \in [T]$ as: 

\begin{align*}
&\E_{\cP} \left[(\bmu^t(x) - \mu(x))^2 -  (\bmu^{t+1}(x) - \mu(x))^2\right]\\
&= \cP_{\cX}(S^t) \E_{\cP}\left[(\bmu^t(x) - \mu(x))^2 -  (\bmu^{t+1}(x) - \mu(x))^2 |x \in S^t\right]\\
&\ge \cP_{\cX}(S^t) \E_{\cP}  \left[(\bmu^t(x) - \mu(x))^2 -  (\bmu^{t}(x) - \eta\lambda^t - \mu(x))^2 |x \in S^t\right]\\
&= \cP_{\cX}(S^t) \E_{\cP}  \left[(\bmu^t(x) - \mu(x))^2 -  \left((\bmu^{t}(x)  - \mu(x))^2 - 2 \eta\lambda^t (\bmu^{t}(x)  - \mu(x)) + (\eta\lambda^t)^2 \right) |x \in S^t\right]\\
&=  \cP_{\cX}(S^t) \E_{\cP}  \left[2 \eta\lambda^t (\bmu^t(x) - \mu(x) )|x \in S^t\right] - \cP_{\cX}(S^t) (\eta\lambda^t)^2 \\
&=  2 \eta \lambda^t \cP_{\cX}(S^t) \left(\bmu^t(S^t) - \mu(S^t)\right) - \cP_{\cX}(S^t) (\eta\lambda^t)^2.
\end{align*}
The inequality would be an equality if we did not project $\bmu^{t}$ into the range $[0,1]$. Performing the projection only decreases its $\ell_2$ distance to $\mu$, which yields the inequality. Rearranging terms and observing that $(\lambda^t)^2 =1$ yields 
\begin{align*}
&\lambda^t \cP_{\cX}(S^t) \left(\bmu^t(S^t) - \mu(S^t)\right) \leq \frac{1}{2\eta} \E_{\cP}\left[(\bmu^t(x) - \mu(x))^2 -  (\bmu^{t+1}(x) - \mu(x))^2\right] + \frac{\eta \cP_{\cX}(S^t)}{2}
\end{align*}
Therefore we have that
\begin{align*}
&\sum_{t=1}^T \lambda^t \cP_{\cX}(S^t) \left(\bmu^t(S^t) - \mu(S^t)\right) \\
\le& \sum_{t=1}^T\left( \frac{1}{2\eta} \E_{\cP}\left[(\bmu^t(x) - \mu(x))^2 -  (\bmu^{t+1}(x) - \mu(x))^2\right] + \frac{\eta \cP_{\cX}(S^t)}{2}\right)\\
=& \frac{1}{2\eta} \E_{\cP} \left[(\bmu^1(x) - \mu(x))^2 -  (\bmu^{T+1}(x) - \mu(x))^2\right] + \frac{\eta}{2}\sum_{t=1}^T \cP_{\cX}(S^t) \\
\le& \frac{1}{2\eta}  + \frac{\eta}{2}\sum_{t=1}^T \cP_{\cX}(S^t)
\end{align*}
as desired. The last inequality follows because $\bmu^1(x), \mu(x)$, and $\bmu^{T+1}(x)$ all fall in $[0,1]$. 
\end{proof}

\thmogdfewrounds*
\begin{proof}
Fix any set  $S \subseteq \cX$ and imagine extending the sequence by setting  $S^{T+1} = S$ and setting $\lambda^{T+1} = \mathrm{sign}( \bmu^{T+1}(S) - \mu(S))$. By Lemma \ref{lem:boundregret}, we would then have:
\begin{align*}
\sum_{t=1}^{T+1}\lambda^t \cP_{\cX}(S^t) \left( \bmu^{t}(S^t) - \mu(S^t) \right)  &\le \frac{1}{2\eta}  + \frac{\eta}{2}\sum_{t=1}^{T+1} \cP_{\cX}(S^t) \\
&\le \frac{1}{2\eta}  + \frac{\eta(T+1)}{2} \\
 &\le \sqrt{T+1} && \text{(substituting $\eta = {1 \over \sqrt{T+1}}$)} 
 \end{align*}
 We can then peel off the last term in the sum corresponding to $S^{T+1} = S$ to obtain:
 \begin{align*}
\cP_{\cX}(S)\left\vert \bmu^{T+1}(S) - \mu(S) \right\vert &\le  \sqrt{T+1} -\sum_{t=1}^{T}\lambda^t \cP_{\cX}(S^t) \left( \bmu^{t}(S^t) - \mu(S^t) \right)\\
&\leq  \sqrt{T+1} - \alpha T  && \text{(by assumption)}\\
&= \alpha && \text{(since  $T=\frac{1}{\alpha^2}-1$)}
\end{align*}
which completes the proof.
\end{proof}

\section{Details and Proofs from Section \ref{subsec:pseudo-moment-calibration}}
For intuition, we can think of the pseudo-moment calibration algorithm as playing the following zero sum game using projected online gradient descent against an adversary who plays best responses. Recall that $\bmu$ is a fixed quantity so that $\tmk$ is well defined.
\begin{align*}
\min_{\bmk} \max_{\substack{R \subseteq \cX\\ \psi \in \{-1, 1\}}} \psi   \cP_\cX (R)  \left(\bmk(R ) - \tmk(R) \right).
\end{align*}

\begin{lemma}
\label{lem:mean-conditioned-moment-gradient}
For any arbitrary $\bmk^1:\cX \to [0,1]$ and any sequence of $(R^t, \psi^t)_{t=1}^T$, we have that 
\[
\sum_{t=1}^T \psi^t \cP_{\cX}(R^t) \left(\bmk^t(R) - \tmk(R) \right) \le \frac{1}{2\eta}  + \frac{\eta}{2}\sum_{t=1}^T\cP_{\cX}(R^t)
\]
\end{lemma}
\begin{proof}
\begin{align*}
&\E_{\cP} \left[(\bmk^t(x) - \tmk(x))^2 -  (\bmk^{t+1}(x) - \tmk(x))^2\right]\\
=& \cP_{\cX}(R^t) \E_{\cP} \left[(\bmk^t(x) - \tmk(x_i))^2 -  (\bmk^{t+1}(x) - \tmk(x))^2 | x \in R^t\right]\\
\ge& \cP_{\cX}(R^t) \E_{\cP} \left[(\bmk^t(x) - \tmk(x))^2 -  (\bmk^{t}(x) - \eta\psi^t - \tmk(x))^2| x \in R^t\right]\\
=& \cP_{\cX}(R^t) \E_{\cP} \left[(\bmk^t(x) - \tmk(x))^2 -  \left((\bmk^{t}(x)  - \tmk(x))^2 - 2 \eta\psi^t (\bmk^{t}(x)  - \tmk(x)) + (\eta\psi^t)^2 \right)| x \in R^t\right]\\
=& 2 \eta\psi^t \cP_{\cX}(R^t) \E_{\cP} \left[\bmk^t(x) - \tmk(x) )| x \in R^t\right] - \cP_{\cX}(R^t)(\eta\psi^t)^2
\end{align*}
Here the inequality comes from the fact that projection can only make the $\ell_2$ norm smaller. Rearranging terms and observing that $(\psi^t)^2 =1$ yields 
\begin{align*}
&\psi^t \cP_{\cX}(R^t) \E_{\cP} \left[\bmk^t(x) - \tmk(x)| x \in R^t\right] \le \frac{1}{2\eta}\E_{\cP} \left[(\bmk^t(x) - \tmk(x))^2 -  (\bmk^{t+1}(x) - \tmk(x_i))^2\right] + \frac{\eta \cP_{\cX}(R^t)}{2}.
\end{align*}
Plugging this inequality back into the regret, we get
\begin{align*}
&\sum_{t=1}^T \psi^t \cP_{\cX}(R^t) \left(\bmk^t(R) - \tmk(R) \right)\\
&=\sum_{t=1}^T \psi^t \cP_{\cX}(R^t) \E_{\cP} \left[\bmk^t(x) - \tmk(x)| x \in R^t\right]\\
&\le \sum_{t=1}^T \left(\frac{1}{2\eta}\E_{\cP} \left[(\bmk^t(x) - \tmk(x))^2 -  (\bmk^{t+1}(x) - \tmk(x_i))^2\right] + \frac{\eta \cP_{\cX}(R^t)}{2}\right) \\
&= \frac{1}{2\eta} \left(\sum_{x \in \cX} \E_{\cP} \left[(\bmk^1(x) - \tmk(x))^2 -  (\bmk^{T+1}(x) - \tmk(x_i))^2\right]\right) + \frac{\eta}{2}\sum_{t=1}^T \cP_{\cX}(R^t) \\
&\le \frac{1}{2\eta}  + \frac{\eta}{2}\sum_{t=1}^T \cP_{\cX}(R^t)
\end{align*}
as desired. The last inequality follows because $\bmk^1(x), \tmk(x)$, and $\bmk^{T+1}(x)$ all fall in $[0,1]$.
\end{proof}

\thmpseudomoment*
\begin{proof}
Set $R^{T+1} = R$. From Lemma \ref{lem:mean-conditioned-moment-gradient}, we get
\begin{align*}
&\sum_{t=1}^{T+1} \psi^t  \cP_{\cX}(R^t) \left(\bmk^t(R^t) - \tmk(R^t)\right) \le \frac{1}{2\eta}  + \frac{\eta}{2}\sum_{t=1}^{T+1} \cP_{\cX}(R^t) \le \sqrt{T+1}\\
\implies &\cP_{\cX}(R^{T+1})\left\vert \bmk^{T+1}(R^{T+1}) - \tmk(R^{T+1}) \right\vert \le \sqrt{T+1} - \beta T = \beta
\end{align*}
\end{proof}

\section{Details and Proofs from Section \ref{sec:finite}}
\begin{theorem}[Chernoff Bound]
\label{thm:chernoff}
Fix distribution $\cP$ and some function $f(x,y) \in [0,1]$. Let $\{(x_b, y_b)\}_{b=1}^n$ be $n$ points sampled i.i.d. from $\cP$. Then, we have for any $\delta \in [0,1]$,
\[
    \Pr_{\{(x_b,y_b)\}_{b=1}^n \sim \cP^n}\left[\left\vert \frac{1}{n}\sum_{b=1}^n f(x_b, y_b) - \E_{(x,y) \sim \cP}[f(x,y)] \right\vert \ge \sqrt{\frac{\ln(\frac{2}{\delta})}{2n}}\right] \le \delta.
\]
\end{theorem}

\begin{lemma}
\label{finitemeasurelem}
For any set $S \subseteq \cX$, 
\[
\Pr_{D \sim \cP^n}\left[\left\vert \frac{1}{n}\sum_{b=1}^n \ind(x_b \in S) - \cP_\cX(S) \right\vert > \sqrt{\frac{\ln(\frac{2}{\delta})}{2n}} \right] \le \delta 
\]
\end{lemma}
\begin{proof}
We apply a Chernoff bound (Theorem \ref{thm:chernoff}) with $f(x,y) = \ind(x \in S)$. Observe that $\E[f(x,y)] = \cP_\cX(S)$. 
\end{proof}

\auditorlem*
\begin{proof} To see this, observe that

\begin{align*}
&\left|\bell(S) - \ell(S)\right| \\
\ge\,\,& \left|\frac{1}{n'}\sum_{b=1}^{n'} \bell(x_b) - \frac{1}{n'} \sum_{b=1}^{n'} \ell(x_b, y_b)\right| - 2\sqrt{\frac{\ln(\frac{2}{\delta})}{2n'}} \\
\ge\,\,&  \frac{\alpha}{\frac{n'}{n} - \sqrt{\frac{\ln(\frac{2}{\delta})}{2n}}}  \\
 \ge\,\,&  \frac{\alpha}{\cP_{\cX}(S)} 
\end{align*}
Here, the first inequality follows from the \eqref{eqn:chernoff2} and \eqref{eqn:chernoff3}, the second from  the condition of Algorithm \ref{alg:auditor-single-set}, and the last inequality follows from \eqref{eqn:chernoff1}.

Finally, if $\frac{1}{n'}\sum_{b=1}^{n'}\bell(x_b) \ge\frac{1}{n'}\sum_{b=1}^{n'}\ell(x_b, y_b)$, then 
\begin{align*}
\bell(S) \ge \frac{1}{n'}\sum_{b=1}^{n'}\bell(x_b) - \sqrt{\frac{\ln(\frac{2}{\delta})}{2n'}} \ge \frac{1}{n'}\sum_{b=1}^{n'}\ell(x_b, y_b)  + \sqrt{\frac{\ln(\frac{2}{\delta})}{2n'}} \ge \ell(S).
\end{align*}
The same argument applies when $\frac{1}{n'}\sum_{b=1}^{n'}\bell(x_b) <\frac{1}{n'}\sum_{b=1}^{n'}\ell(x_b, y_b)$. 
Therefore, $\text{sign}(\frac{1}{n'}\sum_{b=1}^{n'}\bell(x_b) - \frac{1}{n'}\sum_{b=1}^{n'}\ell(x_b, y_b)) = \text{sign}(\bell(S) -\ell(S))$.

\end{proof}

\auditorlemlower*
\begin{proof}
The pre-condition implies that 
\begin{align*}
    &\left|\frac{1}{n'}\sum_{b=1}^{n'} \bell(x_b) - \frac{1}{n'}\sum_{b=1}^{n'} \ell(x_b,y_b) \right| - 2 \sqrt{\frac{\ln(\frac{2}{\delta})}{2n'}}\\
    &\ge \left|\bell(S) - \ell(S)\right| - 4 \sqrt{\frac{\ln(\frac{2}{\delta})}{2n'}} \\
    &\ge \frac{\alpha'}{\cP_\cX(S)} - 4 \sqrt{\frac{\ln(\frac{2}{\delta})}{2n'}}.
\end{align*}
Therefore it is sufficient to show that $\frac{\alpha'}{\cP_\cX(S)} - 4 \sqrt{\frac{\ln(\frac{2}{\delta})}{2n'}} \ge \frac{\alpha}{\frac{n'}{n} - \sqrt{\frac{\ln(\frac{2}{\delta})}{2n}}}$.

\begin{align*}
    \frac{\alpha'}{\cP_\cX(S)}  &= \frac{\alpha + 4\sqrt{\frac{1}{2n}\ln(\frac{2}{\delta})} + \left(\alpha -2\sqrt{\frac{\ln(\frac{2}{\delta})}{2n}}\right)^{-2}\left(2\sqrt{\frac{\ln(\frac{2}{\delta})}{2n}}\right)}{\cP_{\cX}(S)}\\
    &\ge \frac{\alpha + 4\sqrt{\frac{1}{2n}\ln(\frac{2}{\delta})}}{\cP_{\cX}(S) - 2\sqrt{\frac{\ln(\frac{2}{\delta})}{2n}}} &\stepcounter{equation}\tag{\theequation}\label{myeq1}\\
    &\ge \frac{\alpha}{\cP_{\cX}(S) - 2\sqrt{\frac{\ln(\frac{2}{\delta})}{2n}}} + 
    \frac{4\sqrt{\frac{1}{2n}\ln(\frac{2}{\delta})}}{\cP_{\cX}(S) - \sqrt{\frac{\ln(\frac{2}{\delta})}{2n}}}\\
    &\ge \frac{\alpha}{\cP_{\cX}(S) - 2\sqrt{\frac{\ln(\frac{2}{\delta})}{2n}}} + 4 \sqrt{\frac{\ln(\frac{2}{\delta})}{2n (\cP_{\cX}(S) - \sqrt{\frac{\ln(\frac{2}{\delta})}{2n}}) }} &\forall x \in [0,1]:x \le \sqrt{x}\\
    &\ge \frac{\alpha}{\cP_{\cX}(S) - 2\sqrt{\frac{\ln(\frac{2}{\delta})}{2n}}} + 4 \sqrt{\frac{\ln(\frac{2}{\delta})}{2n'}}  & \text{by } \eqref{eqn:chernoff1}\\
    &\ge \frac{\alpha}{\frac{n'}{n} - \sqrt{\frac{\ln(\frac{2}{\delta})}{2n}}} + 4 \sqrt{\frac{\ln(\frac{2}{\delta})}{2n'}} &\text{by } \eqref{eqn:chernoff1}\\
\end{align*}

Inequality \eqref{myeq1} comes from Lemma \ref{lem:cutebound}, where we plug in $x = \cP_\cX(S)$, $c=\alpha + 4\sqrt{\frac{1}{2n}\ln(\frac{2}{\delta})}$ and $\epsilon = 2\sqrt{\frac{\ln(\frac{2}{\delta})}{2n}}$.
\end{proof}

\begin{lemma} \label{lem:cutebound}
For any $0 < \epsilon \le \alpha \le x \le 1$ and $0 < c \le 1$, \[\frac{c + \frac{ \epsilon}{(\alpha-\epsilon)^2}}{x} \ge \frac{c}{x - \epsilon}\]
\end{lemma}

\begin{proof}
Because \[\frac{c + \frac{\epsilon}{(\alpha-\epsilon)^2}}{x} \ge \frac{c}{x} + \frac{\epsilon}{(\alpha-\epsilon)^2},\] it is sufficient to show that 
\[
    \frac{c}{x} + \frac{\epsilon}{(\alpha-\epsilon)^2} \ge \frac{c}{x - \epsilon}.
\]
Because $f(x) = \frac{c}{x}$ is convex, it's easy to see that:
\begin{align*}
    f(x-\epsilon) + \epsilon f'(x-\epsilon) &\le f(x)\\
    \frac{c}{x - \epsilon} - \frac{c \epsilon}{(x -\epsilon)^2} &\le \frac{c}{x}.
\end{align*}
Now, because $\epsilon \le \alpha \le x$ and $0 < c \le 1$, we have
\[
    \frac{c}{x-\epsilon} - \frac{\epsilon}{(\alpha-\epsilon)^2} \le \frac{c}{x}. \qedhere
\] 
\end{proof}

\meanconsistencycor*
\begin{proof}
For each set $S \in \cS$, we write $D_S =\{(x^S_b, y^S_b)\}_{b=1}^{n'_S}$ to denote the points from $D$ that fall in $S$.

First, by union bounding the failure probabilities of Lemma \ref{finitemeasurelem} over every $S \in \cS$, we have with probability $1-\delta|\cS|$, 
\[
    \left\vert \frac{n'_S}{n} - \cP_\cX(S) \right\vert > \sqrt{\frac{\ln(\frac{2}{\delta})}{2n}}.
\]

We apply the Chernoff bound again for every set $S$ where $n'_S > 0$ and take the union bound to argue that with probability at least $1-2|\cS|\delta$, for all such sets $S$ where $n'_S > 0$,
\begin{align*}
&\left\vert \frac{1}{n'_S}\sum_{b=1}^{n'_S} \bell(x^S_b) - \bell(S) \right\vert \le \sqrt{\frac{\ln(\frac{2}{\delta})}{2n'_S}}\\
&\left\vert \frac{1}{n'_S}\sum_{b=1}^{n'_S} \ell(x^S_b,y^S_b) - \ell(S) \right\vert \le \sqrt{\frac{\ln(\frac{2}{\delta})}{2n'_S}}.
\end{align*}

Observe that despite the fact that $n'_S$ is not fixed before we draw the sample, we can still apply a Chernoff bound here because for \emph{every} realized value of $n'_S$, the  distribution, conditional on the value of $n'_S$, of points $(x,y)$ such that $(x,y) \in S$  remains a product distribution, with individual such points  distributed as $\cP|x \in S$.
Now, we go through each scenario:
\begin{enumerate}
    \item \textbf{ConsistencyAuditor outputs some set $S$ and $\lambda$}:
    In this case, $S$ would have been returned only if $n'_S > 0$ due to the if condition in Algorithm \ref{alg:auditor-single-set}. Therefore, $D$ must be approximately close to $\cP$ with respect to $(S,\bell,\ell)$. By Lemma \ref{lem:alpha-prime-auditor}, we have 
    \[
        \vert \bell(S) - \ell(S) \vert \ge \frac{\alpha}{\cP_\cX(S)} \quad\text{and}\quad \lambda = \text{sign}(\bell(S) -\ell(S))
    \]
    \item \textbf{ConsistencyAuditor outputs $NULL$}:
    For any set $S$, $\cP_\cX(S) < \alpha$ directly implies that
    \[
        \vert \bell(S) -\ell(S)\vert \le 1 < \frac{\alpha}{\cP_\cX(S)} \le  \frac{\alpha'}{\cP_\cX(S)}.
    \]
    Therefore, we focus only on sets $S$ where  $\cP_\cX(S) \ge \alpha$. For these sets, we have $n'_S > 0$ because
    \begin{align*}
    \frac{n'_S}{n} \ge \cP_\cX(S) - \sqrt{\frac{\ln(\frac{2}{\delta})}{2n}} \ge \alpha -\sqrt{\frac{\ln(\frac{2}{\delta})}{2n}} > \sqrt{\frac{\ln(\frac{2}{\delta})}{2n}} > 0,
    \end{align*}
    as we assumed $\alpha > 2\sqrt{\frac{\ln(\frac{2}{\delta})}{2n}}$.
    Therefore, for every set $S$ where $\cP_\cX(S) \ge \alpha$, we must have that $D$ must be approximately close to $\cP$ with respect to $(S, \bell, \ell)$. Thus, by applying Lemma \ref{lem:alpha-prime-auditor} to these sets $S$, we have 
    \[
        \left\vert \bell(S) - \ell(S) \right\vert \le \frac{\alpha'}{\cP_\cX(S)},
    \]
    where $\alpha' = \alpha + 4\sqrt{\frac{1}{2n}\ln(\frac{2}{\delta})} + \left(\alpha -2\sqrt{\frac{\ln(\frac{2}{\delta})}{2n}}\right)^{-2}\left(2\sqrt{\frac{\ln(\frac{2}{\delta})}{2n}}\right)$. 
\end{enumerate}
\end{proof}

\finitewrappertheorem*
\begin{proof}
We prove each guarantee in turn.

\medskip
\noindent \textbf{Total Iterations}: 
As argued in Theorem \ref{thm:wrapper-distr-known}, if the auditor can successfully find a set $S$ on which there is $\alpha$-mean inconsistency and $\beta$-pseudo-moment inconsistency respectively in AlternatingGradientDescentFinite (Algorithm \ref{alg:wrapper-finite}) and PseudoMomentConsistencyFinite (Algorithm \ref{alg:moment-finite}), Theorem \ref{thm:mean-calibration-convergence} guarantees that $T$ will be at most $\frac{1}{\alpha^2}-1$ and Theorem \ref{thm:pseudo-moment-calibration} guarantees that the total number of gradient descent operations in each PseudoMomentconsistencyFinite will be at most $\frac{1}{\beta^2}-1$. Then, because in each iteration of AlternatingGradientDescentFinite, there are $k-1$ calls to PseudoMomentconsistencyFinite, the total number of number of gradient descent operations will be at most $Q=Q_\alpha(1 + Q_\beta)$ where $Q_\alpha = \frac{1}{\alpha^2}-1$ and $Q_\beta = (k-1)(\frac{1}{\beta^2}-1)$. 

Therefore, it is sufficient for us to show that there is $\alpha$-mean inconsistency and $\beta$-pseudo-moment inconsistency on every $S^t$ and $R$ returned by ConsistencyAuditor (Algorithm \ref{alg:consistency-auditor}) for AlternatingGradientDescentFinite and PseudoMomentconsistencyFinite respectively.

For AlternatingGradientDescentFinite, because we set $\bell^t(x) = \bmu^t(x)$, $\ell(x,y) = y$, and  $\cS^t =  \{G(\bmu, i): G \in \cG, i \in [m]\} \cup \{G(\bmu,\bma,i,j): G \in \cG, i,j \in [m], a \in \{2, \dots, k\}\}$,  Corollary \ref{cor:consistency-auditor}  guarantees  that with probability $1-3\delta|\cG|(m^2 + m)$, $\bmu^t$ is $\alpha$-mean inconsistent on $S^t$ with n $\lambda^t = \text{sign}(\bmu^t(S^t) - \mu(S^t))$ as desired. Because $T$ is at most $Q_\alpha$, by a union bound, $\bmu^t$ is $\alpha$-mean-inconsistent on $S^t$ for every $t \in [T]$ with probability $1-3\delta|\cG|(m^2 + m)Q_\alpha$. 

Likewise, for PseudoMomentconsistencyFinite, we set $\bell(x) = \bmk(x)$, $\ell(x,y) = (y-\bmu(x))^a$, and  $\cS = \{G(\bmu,\bma,i,j): G \in \cG, i,j \in [m], a \in \{2, \dots, k\}\}$. Hence, by union bounding over every $a \in \{2, \dots, k\}$, Corollary \ref{cor:consistency-auditor} promises us that with probability $1-3\delta|\cG|m^2Q_\alpha Q_\beta$, $\bma$ is $\beta$-pseudo-moment inconsistent on $R$ throughout every iteration of
PseudoMomentConsistencyFinite for every $a \in \{2, \dots, k\}$ and  $\psi=\text{sign}(\bma(S) - \tma(S)$ as desired. Note that there are a total of $Q_\beta$ calls to ConsistencyAuditor from each PseudoMomentConsistencyFinite, which is invoked a total of $Q_\alpha$ many times.

\medskip

\noindent \textbf{Mean Multi-Calibration}:
Our algorithm halts only if ConsistencyAuditor doesn't find $S$ in AlternatingGradientDescentFinite. Corollary \ref{cor:consistency-auditor} promises us that with probability $1-3\delta (m^2+m)|\cG|$, $\bmu^T$ must be $\alpha'$-mean-consistent on every set $S \in \cS^T$. Because $\cS$ includes $\{G(\bmu^T, i): G \in \cG, i \in [m]\}$, it must be that $\bmu^T$ is $\alpha'$-mean multi-calibrated with respect to $\cG$.
 
 \medskip
 
 \noindent \textbf{Mean Conditioned Moment Multi-Calibration}: In the last round $T$, consider each $\bma^T$ for $a \in \{2, \dots, k\}$. PseudoMomentConsistencyFinite returns $\bma^T$ only if ConsistencyAuditor doesn't return any $R$. Corollary \ref{cor:consistency-auditor} guarantees us that with probability $1-3\delta m^2 |\cG|$, $\bma^T$ must be $\beta'$-pseudo-moment-consistent. Because $\bmu^T$ is $\alpha'$-mean consistent and $\bma^T$ is pseudo-moment-consistent on $\{G(\bmu,\bma,i,j): G \in \cG, i,j \in [m], a \in \{2, \dots, k\}\}$, Lemma \ref{lem:fine-consist-mean-conditioned-moment-cal} tells us that $(\bmu^T, \bma^T)$ must be $(\alpha', a\alpha +\beta', \frac{a}{m})$-mean-conditioned-moment multicalibrated. By union bounding over each $a \in \{2, \dots, k\}$ the total failure probability is $1-3\delta|\cG|( km^2  +m)$.
 \end{proof}
 
 \finitewrappercor*
 \begin{proof}
Note that by construction, we have 
\[
    \alpha > 2\sqrt{\frac{\ln(\frac{2}{\delta})}{2n}}\quad\text{and}\quad \beta > 2\sqrt{\frac{\ln(\frac{2}{\delta})}{2n}}.
\] 
Therefore, in Theorem \ref{thm:finite-wrapper}, the level of mean calibration for $\bmu^T$ will be 
\begin{align*}
&\alpha +4\sqrt{\frac{\ln\left(\frac{2}{\delta}\right)}{2n}} + \left(\alpha -2\sqrt{\frac{\ln\left(\frac{2}{\delta}\right)}{2n}}\right)^{-2}\left(2\sqrt{\frac{\ln(\frac{2}{\delta})}{2n}}\right) \\
&\le \alpha +4\sqrt{\frac{\ln\left(\frac{2\bQ}{\delta}\right)}{2n_\alpha}} + \left(\alpha -2\sqrt{\frac{\ln\left(\frac{2\bQ}{\delta}\right)}{2n_\alpha}}\right)^{-2}\left(2\sqrt{\frac{\ln(\frac{2\bQ}{\delta})}{2n_\alpha}}\right) \\
&= 2\sqrt{\frac{\ln(\frac{2\bQ}{\delta})}{2n_\alpha}} + \epsilon +4\sqrt{\frac{\ln\left(\frac{2\bQ}{\delta}\right)}{2n_\alpha}} + \frac{\left(2\sqrt{\frac{\ln(\frac{2\bQ}{\delta})}{2n_\alpha}}\right)}{\epsilon^2} \\
&= \sqrt{\frac{\ln(\frac{2\bQ}{\delta})}{2n_\alpha}}\left(6 + \frac{2}{\epsilon^2}\right) + \epsilon\\
&= \alpha',
\end{align*}
where the first inequality follows because $n \ge \frac{\ln(\frac{2\bQ}{\delta})}{\ln(\frac{2}{\delta})} n_\alpha$ and the last equality from the definition of $n_\alpha$.

Applying the same analysis, we can show that we satisfy pseudo-moment-consistency at level $\beta'$. Therefore, for any $a \in \{2, \dots, k\}$, $(\bmu^t, \bma^T)$ satisfy $(\alpha', a\alpha' + \beta', \frac{a}{m})$-mean-conditioned-moment multicalibration.

The failure probabilities for mean muticalibration and that of mean-conditioned-moment multicalibration are both less than $\delta'$, as $3\delta (m^2+m)|\cG| \le 3\delta|\cG|( km^2  +m) \le \delta'$ and $3\delta|\cG|( km^2  +m) \le \delta'$.

The failure probability for termination is 
\begin{align*}
    &3\delta|\cG|\left(\frac{1}{\alpha^2} - 1\right)\left((m^2 + m)+ m^2(k-1)\left(\frac{1}{\beta^2} - 1\right)\right)\\ 
    &\le 3\delta|\cG|\left(\frac{1}{\alpha^2}\right)
    \frac{2km^2}{\beta^2}\\
    &= 6|\cG|km^2\delta \cdot \frac{1}{\left(\sqrt{\frac{\ln(\frac{2\bQ}{\delta})}{2n_\alpha}}+\epsilon\right)^2 \left(\sqrt{\frac{\ln(\frac{2\bQ}{\delta})}{2n_\beta}}+\epsilon\right)^2}\\
    &\le 6|\cG|km^2\delta \cdot \frac{1}{\frac{\ln(\frac{2\bQ}{\delta})}{2n_\alpha} \frac{\ln(\frac{2\bQ}{\delta})}{2n_\beta}}\\
    &\le 24 |\cG|km^2\delta \cdot \frac{1}{\left(\ln(\frac{2\bQ}{\delta})\right)^2} n_\alpha n_\beta\\
    &= 24 |\cG|km^2\delta \cdot \frac{1}{\left(\ln(\frac{2\bQ}{\delta})\right)^2}
    \frac{\ln(\frac{2\bQ}{\delta})}{2\left(\frac{\alpha' - \epsilon}{6+\frac{2}{\epsilon^2}}\right)^2}
    \frac{\ln(\frac{2\bQ}{\delta})}{2\left(\frac{\beta' - \epsilon}{6+\frac{2}{\epsilon^2}}\right)^2}\\
    &\le 6 |\cG|km^2\delta \cdot
    \frac{1}{\left(\frac{\alpha' - \epsilon}{6+\frac{2}{\epsilon^2}}\right)^2\left(\frac{\beta' - \epsilon}{6+\frac{2}{\epsilon^2}}\right)^2}\\
    &= \bQ \delta \\
    &\le \delta'
\end{align*}
\end{proof}

\runtimethm*
\begin{proof}
Theorem \ref{thm:finite-wrapper} tells us that except with probability $1-3\delta|\cG|Q_\alpha\left((m^2 + m)+ m^2Q_\beta\right)$, the algorithm will halt after at most $Q$ many gradient descent updates. For each gradient descent update, it must have been that Algorithm \ref{alg:consistency-auditor} was invoked with either $\cS =  \{G(\bmu, i): G \in \cG, i \in [m]\} \cup \{G(\bmu,\bma,i,j): G \in \cG, i,j \in [m], a \in \{2, \dots, k\}\}$ or $\cS =  \{G(\bmu,\bma,i,j): G \in \cG, i,j \in [m], a \in \{2, \dots, k\}\}$. Note that Algorithm \ref{alg:consistency-auditor} needs to iterate through each set $S$ in $\cS$, whose size is at most $O(|\cG|m^2)$ in either case. And processing each set $S$ through Algorithm \ref{alg:auditor-single-set} requires finding the average of at most $O(n)$ elements twice. Therefore, the algorithm will take time $O\left(Q |\cG| m^2 n\right)$ with probability $1-3\delta|\cG|Q_\alpha\left((m^2 + m)+ m^2Q_\beta\right)Q$.
\end{proof}

 \begin{algorithm}[H]
\begin{algorithmic}
\STATE $D=\{(x_b, y_b)\}_{b=1}^n \sim \cP^n$ 
\STATE $D^{\text{check}} \sim \cP^n$ 
\STATE $\cS = \{\cX(\bmu,\bma,i,j): i,j \in [m]\}$
\STATE $\bell(x) = \bma(x)$
\STATE $\ell(x,y) = (y- \bmu(x))^a$
\STATE $R,\psi = \text{LearningOracleConsistencyAuditorWrapper}(\bell, \ell, \beta, \delta, D, D^{\text{check}}, \cS, A)$
    \WHILE{$R,\psi \neq NULL$}
            \STATE $\bma(x) = \begin{cases}
            	\text{project}_{[0,1]}(\bma(x)  - \beta \psi) & \text{if } x \in R \\
            	\bma(x) & \text{otherwise.}
            	\end{cases}$
            \STATE $D=\{(x_b, y_b)\}_{b=1}^n \sim \cP^n$ 
            \STATE $\cS = \{\cX(\bmu,\bma,i,j): i,j \in [m]\}$
            \STATE $\bell(x) = \bma(x)$
            \STATE $D=\{(x_b, y_b)\}_{b=1}^n \sim \cP^n$ 
\STATE $D^{\text{check}} \sim \cP^n$ 
            \STATE $R,\psi = \text{LearningOracleConsistencyAuditorWrapper}(\bell, \ell, \beta, \delta, D, D^{\text{check}}, \cS, A)$
\ENDWHILE
\STATE return $\bmk$
\end{algorithmic}
\caption{$\text{PseudoMomentConsistencyWithOracle}(a, \beta, \delta, \bmu, \bma, n, \cG)$}
\label{alg:moment-finite-oracle}
\end{algorithm}
 
 \begin{algorithm}
\begin{algorithmic}
\STATE initialize $\bmu^1(x) = 0$ for all $x$
\STATE for all $1 < a \leq k$, initialize $\bma^1(x) = 0$ for all $x$
\STATE $t=1$
\STATE $\bell^t(x) = \bmu(x)$
\STATE $\ell(x,y) = y$
\STATE $D^t= \sim \cP^n$
\STATE ${D^{\text{check}}}^t \sim \cP^n$
\STATE $\cS^t = \{\cX(\bmu^t, i): i \in [m]\} \cup \{\cX(\bmu^t,\bma^t,i,j): i,j \in [m], a \in \{2, \dots, k\}\}$
\STATE $S^t,\lambda^t = \text{LearningOracleConsistencyAuditorWrapper}(\bell^t, \ell, \beta, \delta, D^t, {D^{\text{check}}}^t, \cS^t, A)$
\WHILE{$S^t,\lambda^t \neq NULL$}
    \STATE $\bmu^{t+1} = \text{MeanConsistencyUpdate}( \bmu^{t}, S^t,\lambda^t)$
    \FOR{$a=2, \dots, k$}
        \STATE $\bma^{t+1} = \text{PseudoMomentConsistencyFinite}(a, \beta, \delta, \bmu^{t+1}, \bma^{t}, n, \cG)$.
    \ENDFOR
    \STATE $t = t+1$
    \STATE $\bell^t(x) = \bmu(x)$
    \STATE $D^t \sim \cP^n$ 
    \STATE ${D^{\text{check}}}^t \sim \cP^n$
    \STATE $\cS^t = \{\cX(\bmu^t, i): i \in [m]\} \cup \{\cX(\bmu^t,\bma^t,i,j): i,j \in [m], a \in \{2, \dots, k\}\}$
    \STATE $S^t, \lambda^t = \text{LearningOracleConsistencyAuditorWrapper}(\bell^t, \ell, \beta, \delta, D^t, {D^{\text{check}}}^t, \cS^t, A)$
\ENDWHILE
\STATE return $(\bmu^t, \{\bma^t\}_{a=2}^k)$
\end{algorithmic}
\caption{$\text{AlternatingGradientDescentWithOracle}(\alpha,\beta, \delta, n, \cG, A)$}
\label{alg:wrapper-finite-oracle}
\end{algorithm}

\oraclelem*
\begin{proof}
\begin{align*}
\E_{(x,y)}[\chi_S(x) \cdot r_R^+(x,y)] &= \sum_{x,y} \cP(x,y) \chi_S(x) r_R^+(x,y)\\
&= \sum_{x,y: x \in S, x \in R} \cP(x,y) (\bell(x)-\ell(x,y))\\
&=\cP_\cX(R \cap S) \left(\bell(R \cap S) - \ell(R \cap S)\right)
\end{align*}
The same argument applies for $r^-_R$ as well.
\end{proof}

\oracleauditor*
\begin{proof}
With probability at least $1-3|\cR|\delta$ (since $|\cR| \ge |\cV|$), Corollary \ref{cor:consistency-auditor}  gives us the following guarantees:
\begin{enumerate}
    \item If $S,\lambda$ is returned, then \[
    \left\vert \bell(S) - \ell(S) \right\vert \ge \frac{\alpha}{\cP_\cX(S)}
    \quad\text{and}\quad \lambda = \text{sign}(\bell(S) - \ell(S)).
    \]
    \item If $NULL$ is returned, then for all $V \in \cV$,
    \begin{align}
        \vert \bell(V) - \ell(V) \vert \le \frac{\alpha'}{\cP_\cX(V)}
        \label{eqn:cor-alpha-prime-guarantee}
    \end{align}
\end{enumerate}
Whenever the distributional closeness conditions of Definition \ref{def:chernoff-bounds-closeness} hold (which occur with the same  $1-3|\cR|\delta$ success probability of Corollary \ref{cor:consistency-auditor}),  and $\alpha > 2\sqrt{\frac{\ln(\frac{2}{\delta})}{2n}}$, it must be that if $|D_R| = 0$  then $\cP_\cX(R) \le \alpha$. And for any such $R$ we have that $\cP_\cX(R \cap S') \le \alpha$ for any other set $S'$, which implies that we trivially  satisfy $(\alpha' + \rho)$-mean consistency for $R \cap S'$. More precisely, if $\cP_\cX(R) \le \alpha$, then
\[  
    \sup_{\chi_{S'} \in \cH} \vert \bell(R \cap S') - \ell(R \cap S')\vert \le 1 \le \frac{\alpha}{\cP_\cX(R \cap S')} \le \frac{\alpha' + \rho}{\cP_\cX(R \cap S')}.
\]

We can therefore restrict our attention to those  $R \in \cR$ that satisfy $|D_R|>0$, since we only have a non-trivial statement to prove for sets $R$ with  $\cP_\cX(R) > \alpha$.  Using Lemma \ref{lem:oracle-guarantee} and the definition of an agnostic learning oracle, we know that for each $V = R \cap S^+$, with probability $1-p(n)$,
\begin{align}
&\cP_\cX(R \cap S^+) \left(\bell(R \cap S^+) - \ell(R \cap S^+)\right) + \rho \nonumber\\
&= \E_{(x,y)}[\chi_{S^+}(x) \cdot r_R^+(x,y)] + \rho \nonumber\\
&\ge \sup_{\chi_{S'} \in \cH}\E_{(x,y)}[\chi_{S'}(x) \cdot r_R^+(x,y)] \nonumber\\
&= \sup_{\chi_{S'} \in \cH}\cP_\cX(R \cap S') \left(\bell(R \cap S') - \ell(R \cap S')\right) \label{eqn:oracle-positive-residual}
\end{align}
The same argument applies for $V=R \cap S^-$, and we obtain 
\begin{align}
    \cP_\cX(R \cap S^+) \left(\ell(R \cap S^-) - \bell(R \cap S^-)\right) + \rho 
    \ge \sup_{\chi_{S'} \in \cH}\cP_\cX(R \cap S') \left(\ell(R \cap S') - \bell(R \cap S')\right)\label{eqn:oracle-negative-residual}.
\end{align}

Combining \eqref{eqn:cor-alpha-prime-guarantee}, \eqref{eqn:oracle-positive-residual}, and \eqref{eqn:oracle-negative-residual}, we get that with probability $1-2|\cR|p(n)$, 
\[  
    \sup_{\chi_{S'} \in \cH} \vert \bell(R \cap S') - \ell(R \cap S')\vert \le \frac{\alpha' + \rho}{\cP_\cX(R \cap S')}
\]

\end{proof}

\oracleruntime*
\begin{proof}
The running time of Algorithm \ref{alg:consistency-auditor-learning-oracle-wrapper} is $O(m^2\tau(n))$, as we always call it with $|\cR|= O(m^2)$ and we assumed $\tau(n) = \Omega(n)$, meaning the running time of the learning oracle dominates the  calculations in the empirical check. And Theorem \ref{thm:agnostic} gives that with probability $1-3\delta Q_\alpha\left((m^2 + m)+ m^2Q_\beta\right)$, there will be at most $Q$ gradient descent operations. Because the number of gradient descent operations is equal to the number of subroutine calls to Algorithm \ref{alg:consistency-auditor-learning-oracle-wrapper}, the overall running time is $O(Qm^2\tau(n))$.
\end{proof}

\section{A Submodular Set-Cover Formulation} \label{app:formulation}
We can define the following problem. Theorem \ref{thm:ciguarantee} shows us that for every even $a$, and every $G \in \cG$, $i,j \in [m]$, $I_{\gamma,a}(x)$ forms a valid marginal prediction interval for every set $G(\bmu, \bma, i,j)$ with  probability at least $\gamma$ under $\cP_\cX$. Can we  construct tighter prediction intervals using all $\lfloor\frac{k}{2}\rfloor$ moments?

We make the following simplifying assumptions in this section:
\begin{enumerate}
    \item $\cX$ is a set with finite cardinality. 
    \item  $\cP_\cX$ is known exactly (note that we do not assume we know the distribution on \emph{labels} $y$, which preserves the core motivation of the problem).
    \item For every $x \in \cX$, there exists $G \in \cG$, $a$ even s.t. $1<a \leq k$, and $i,j \in [m]$ such that $x \in G(\bmu, \bma, i,j)$ and $\cP_{\cX}(G(\bmu, \bma, i,j)) \geq \gamma$ (otherwise there is no way to give a valid marginal prediction interval for such an $x$).
\end{enumerate}

Let us define the set of all relevant sets as $$\cS \equiv \{G(\bmu, \bma, i,j): \forall G \in \cG, i,j \in [m], 1 < a \leq k, a \text{ even s.t. } \cP_{\cX}(G(\bmu, \bma, i,j)) \geq \delta\}.$$ With each set $S \in \cS$, we associate the width $\Delta_S(\cdot)$ in the obvious way. 

Given any $\cS' \subseteq \cS$ we say that $\cS'$ covers $\cX$ if $\forall x \in \cX$, $\exists S \in \cS'$ s.t. $x \in S$. Given any $\cS' \subseteq \cS$ that covers $\cX$, we can construct valid marginal prediction intervals for all $x \in X$:
\begin{align*}
&\Delta_{\cS'}(x) \equiv \max_{S \in \cS'| x \in S} \Delta_S(x),\\
&I_{\cS'}(x) = [\bmu(x) - \Delta_{\cS'}(x), \bmu(x) + \Delta_{\cS'}(x)]. 
\end{align*}

To see that this will result in a valid prediction interval, observe that for any $x \in X$, it is covered by some $S \in \cS'$. By definition of $\cS'$, $S = G(\bmu, \bma, i,j)$ for some $a$ even, $i,j \in [m]$, $G \in \cG$. Note that $I_{\gamma,a}(x) \subseteq I_{\cS'}(x)$ by construction of $I_{\cS'}(\cdot)$. Therefore Theorem \ref{thm:ciguarantee} ensures that these prediction intervals are valid for any $S \in \cS'$, and indeed, therefore valid for any group $G \in \cG$.

A natural optimization problem is to find a subset $\cS'$ that covers $\cX$ so as to minimize the expected width of the marginal prediction intervals that can be produced in this way, i.e. solves (exactly or approximately)
\begin{align*}
    &\min_{\cS' \subseteq \cS} \E_{x \sim \cP_{\cX}}[\Delta_{\cS'}(x)]\\
    \text{s.t. }& \cS' \text{ covers } \cX.
\end{align*}

We can rewrite the problem in the following way. Let $A$ be a $0-1$ matrix of dimension $|\cX| \times |\cS|$. The columns correspond to sets $S \in \cS$ and the rows to elements $x \in \cX$. If $A_{xS} =1$ this means that element $x \in \cX$ is contained in set $S \in \cS$. Associated with each column $S$ there is a function $\Delta_S$. Recall that $\cP_{\cX}(x)$ denotes the probability of $x$.

We can denote any subset of $\cS' \subseteq \cS$ by a $0/1$ vector $w \in \{0,1\}^{|\cS|}$ such that $w_S =1$ if $S \in \cS'$.  We can therefore recast the optimization problem:
\begin{align*}
\min_{z,w} &\sum_{x \in \cX} \cP_{\cX}(x) z_x \\
\text{s.t. }& z_x - \Delta_S(x) A_{xS} w_S \geq 0 &\forall x \in \cX\\
& \sum_{S \in \cS} A_{xS} w_{S} \geq 1 &\forall x \in \cX.
\end{align*}
 
For any subset $\cS'\subseteq \cS$ let $f_x(\cS') = \max_{S \in \cS'} \Delta_S(x) A_{xS}$. Notice $f_x(\cS')$ is a non-decreasing and submodular function of $\cS'$. Let $f(\cS') = \sum_x \cP_{\cX}(x) f_x(\cS')$, clearly $f(\cdot)$ is non-decreasing and submodular. Similarly, for any $\cS' \subseteq \cS$ let $w$ be the associated $0/1$ vector and define $g(\cS') = |\{x: \sum_{S \in \cS} A_{xS} w_{S} \geq 1\}|$. Again $g(\cdot)$ is a non-decreasing and submodular function of $\cS'$. Observe that can write our problem as:
\begin{align*}
\min_{\cS' \subseteq \cS} & \,\,f(\cS'),\\
\text{s.t. }& g(\cS') \geq |\cX|.
\end{align*}
Therefore our problem is the submodular cost submodular cover problem. 

We can now hope to apply known results to solve it.  For example, \cite{wan2010greedy} show that the greedy solution (to iteratively add the set with the smallest average width) is approximately optimal. In particular, their Theorem 2.1  guarantees that the greedy solution provides a $\frac{k}{2} H$-approximate to our optimization problem where $H$ is the $\ell\tth$-harmonic number, $\ell = \max \{|S|: S \in \cS\}$, and $k$ is the number of moments we have access to. 

Unfortunately in this context, these guarantees are unsatisfactory: the approximation grows with the number of moments $k$ we have access to, and with $\log|\cX|$, which will typically be linear in the data dimension. Note that  \cite{wan2010greedy} study general submodular objective functions and does not exploit the specific structure of the objective function here. We leave the question of whether guarantees can be offered for this problem to future research. Another natural question is how to approximate this optimization when $\cP_\cX$ is not known, i.e. we only have a finite sample from $\cP$.  

\end{document}